\def\isReadyToSubmit{0}   % mark 1 before submission (draft or camera ready)
\def\isAnonymousSubmission{0}  % mark 1 if this is an anonymous submission
\newcommand{\subparagraph}{}
\documentclass[10pt,conference, compsocconf]{IEEEtran}

\usepackage{times}
\usepackage{graphicx}
\usepackage{subfigure}
\usepackage{array} % used for tabular environments
\usepackage{amsmath,amsfonts,amssymb,amsthm}
\usepackage{color}
\usepackage[hyphens]{url}
\usepackage{listings} % used for source code
\usepackage{xspace} % package that may or may not insert "eaten" spaces.
\usepackage{algpseudocode} % used for inserting pseudocode
\usepackage{algorithm} % used for typesetting psuedocode.
\usepackage{wrapfig} % wrap text around figures
\usepackage{flushend} % don't leave a blank column
\usepackage{multirow} % nested tree tables
\usepackage{booktabs} % optimizations for tables
\usepackage{anyfontsize} % select font sizes
\usepackage{paralist} % itemize things within paragraphs
\usepackage[font=small]{caption}
\usepackage{courier} % proper bold version
\usepackage{verbatim}
\usepackage[usenames,dvipsnames,svgnames,table]{xcolor}
\usepackage{calc}
\usepackage{amsthm}
\usepackage{pbox}
\usepackage{algorithm, algpseudocode}
\usepackage{tcolorbox}

\usepackage[compact,small]{titlesec}
\titlespacing{\section}{0pt}{*.6}{*.2}
\titlespacing{\subsection}{0pt}{*.4}{*.2}
\titlespacing{\subsubsection}{0pt}{*.4}{*.2}

\usepackage{etoolbox}
\makeatletter
\patchcmd{\ttlh@hang}{\parindent\z@}{\parindent\z@\leavevmode}{}{}
\patchcmd{\ttlh@hang}{\noindent}{}{}{}
\makeatother

\let\oldenumerate\enumerate
\renewcommand{\enumerate}{
  \oldenumerate
  \setlength{\itemsep}{.0pt}
  \setlength{\parskip}{0pt}
  \setlength{\parsep}{0pt}
}

\let\olditemize\itemize
\renewcommand{\itemize}{
  \olditemize
  \setlength{\itemsep}{1pt}
  \setlength{\parskip}{0pt}
  \setlength{\parsep}{0pt}
}

\newtheorem{proposition}{Proposition}
\newtheorem*{proposition*}{Proposition}
\newtheorem{lemma}{Lemma}
\newtheorem*{lemma*}{Lemma}
\newtheorem{corollary}{Corollary}
\newtheorem*{corollary*}{Corollary}
\theoremstyle{definition}

\makeatletter
\renewcommand{\ALG@beginalgorithmic}{\footnotesize}
\makeatother

\def\F{Fig.~}
\renewcommand{\figurename}{Fig.}

\def\E{\mathbb{E}}

\newcommand{\ar}[3]{} %% something bibtex is doing - ignore it
\include{latexdefs}
\include{mathnotation}

\newcommand{\heading}[1]{{\vspace{0pt}\noindent\bf{#1}}} % inside section
{\makeatletter
 \gdef\xxxmark{%
   \expandafter\ifx\csname @mpargs\endcsname\relax % in minipage?
     \expandafter\ifx\csname @captype\endcsname\relax % in figure/caption?
       \marginpar{\textcolor{red}{xxx~}}% not in a caption or minipage, can use marginpar
     \else
       \textcolor{red}{xxx~}% notice trailing space
     \fi
   \else
     \textcolor{red}{xxx~}% notice trailing space
   \fi}
 \gdef\xxx{\@ifnextchar[\xxx@lab\xxx@nolab}
 \long\gdef\xxx@lab[#1]#2{{\bf [\xxxmark \textcolor{red}{#2} ---{\sc #1}]}}
 \long\gdef\xxx@nolab#1{{\bf [\xxxmark \textcolor{red}{#1}]}}
 \ifnum\isReadyToSubmit=1
   \long\gdef\xxx@lab[#1]#2{}\long\gdef\xxx@nolab#1{}
 \fi
}

{\makeatletter
 \gdef\edit{\@ifnextchar[\edit@lab\edit@nolab}
 \long\gdef\edit@lab[#1]#2{[\textcolor{red}{#2} ---{\sc #1}]}
 \long\gdef\edit@nolab#1{[\textcolor{red}{#1}]}
 \ifnum\isReadyToSubmit=1
   \long\gdef\edit@lab[#1]#2{[#2]}
 \fi
}

\newcommand{\ignore}[1]{}

\definecolor{grey}{rgb}{0.5,0.5,0.5}

\usepackage{bbm}
\usepackage{setspace}

\renewcommand{\tabcolsep}{1.5mm}

\def\pixeldp{PixelDP\xspace}

\def\Lname{construction attack bound\xspace}
\def\L{L\xspace}
\def\Tname{prediction robustness threshold\xspace}
\def\T{T\xspace}
\def\Lattackname{empirical attack bound\xspace}

\begin{document}

\title{Certified Robustness to Adversarial Examples with Differential Privacy}

\ifnum\isAnonymousSubmission=0
\author{ {\rm Mathias Lecuyer, Vaggelis Atlidakis, Roxana Geambasu, Daniel Hsu, and Suman Jana} \\
         Columbia University
       }
\else
  \author{ {} }
\fi

\maketitle

\thispagestyle{plain}
\pagestyle{plain}

% -------------------- %
\begin{abstract}
Adversarial examples that fool machine learning models, particularly deep neural networks, have been a topic of intense research interest, with attacks and defenses being developed in a tight back-and-forth.
Most past defenses are {\em best effort} and have been shown to be vulnerable to sophisticated attacks.
Recently a set of {\em certified defenses} have been introduced, which provide guarantees of robustness to norm-bounded attacks. However these defenses either do not scale to large datasets or are limited in the types of models they can support.
This paper presents the first certified defense that both scales to large networks and datasets (such as Google's Inception network for ImageNet) and applies broadly to arbitrary model types.
Our defense, called {\em \pixeldp}, is based on a novel connection between robustness against adversarial examples and differential privacy, a cryptographically-inspired privacy formalism, that provides a rigorous, generic, and flexible foundation for defense.
\end{abstract}

\section{Introduction}
\label{sec:introduction}

Deep neural networks (DNNs) perform exceptionally well on many machine learning tasks, including safety- and security-sensitive applications such as self-driving cars~\cite{DBLP:journals/corr/BojarskiTDFFGJM16}, malware classification~\cite{pascanu2015malware}, face recognition~\cite{parkhi2015deep}, and critical infrastructure~\cite{Zohrevand:2017:DLB:3132847.3133031}.
Robustness against malicious behavior is important in many of these applications, yet in recent years it has become clear that DNNs are vulnerable to a broad range of attacks.
Among these attacks -- broadly surveyed in~\cite{sok-towards-the-science-of-security-and-privacy} -- are {\em adversarial examples}:  the adversary finds small perturbations to correctly classified inputs that cause a DNN to produce an erroneous prediction, possibly of the adversary's choosing~\cite{szegedy2013intriguing}.
Adversarial examples pose serious threats to security-critical applications.
A classic example is an adversary attaching a small, human-imperceptible sticker onto a stop sign that causes a self-driving car to recognize it as a yield sign.
Adversarial examples have also been demonstrated in domains such as reinforcement learning~\cite{kos2017delving} and generative models~\cite{kos2017adversarial}.

Since the initial demonstration of adversarial examples~\cite{szegedy2013intriguing}, numerous attacks and defenses have been proposed, each building on one another.
Initially, most defenses used {\em best-effort} approaches and were broken soon after introduction.
Model distillation, proposed as a robust defense in~\cite{papernot2016distillation}, was subsequently broken in~\cite{carlini-attacks}. Other work~\cite{lu2017no} claimed that adversarial examples are unlikely to fool machine learning (ML) models in the real-world, due to the rotation and scaling introduced by even the slightest camera movements.
However, \cite{athalye2017synthesizing} demonstrated a new attack strategy that is robust to rotation and scaling.
While this back-and-forth has advanced the state of the art, recently the community has started to recognize that rigorous, theory-backed, defensive approaches are required to put us off this arms race.

Accordingly, a new set of {\em certified defenses} have emerged over the past year, that provide rigorous guarantees of robustness against norm-bounded attacks~\cite{pmlr-v70-cisse17a, raghunathan2018certified, wong2018provable}.
These works alter the learning methods to both optimize for robustness against attack at training time and permit provable robustness checks at inference time.
At present, these methods tend to be tied to internal network details, such as the type of activation functions and the network architecture. They struggle to generalize across different types of DNNs and have only been evaluated on small networks and datasets.

We propose a new and orthogonal approach to certified robustness against adversarial examples that is {\em broadly applicable, generic, and scalable}.
We observe for the first time a connection between {\em differential privacy} (DP), a cryptography-inspired formalism, and a definition of robustness against norm-bounded adversarial examples in ML.
We leverage this connection to develop {\em \pixeldp}, the first certified defense we are aware of that both scales to large networks and datasets (such as Google's Inception network trained on ImageNet) and can be adapted broadly to arbitrary DNN architectures.
Our approach can even be incorporated with no structural changes in the target network (e.g., through a separate auto-encoder as described in Section~\ref{sec:noise-layer}).
We provide a brief overview of our approach below along with the section references that detail the corresponding parts.

\S\ref{sec:dp-robustness-connection} establishes the DP-robustness connection formally (our first contribution). To give the intuition, DP is a framework for randomizing computations running on databases such that a small change in the database (removing or altering one row or a small set of rows) is guaranteed to result in a bounded change in the distribution over the algorithm's outputs.
Separately, robustness against adversarial examples can be defined as ensuring that small changes in the input of an ML predictor (such as changing a few pixels in an image in the case of an $\l_0$-norm attack) will not result in drastic changes to its predictions (such as changing its label from a stop to a yield sign).
Thus, if we think of a DNN's inputs (e.g., images) as databases in DP parlance, and individual features (e.g., pixels) as rows in DP, we observe that randomizing the outputs of a DNN's prediction function to enforce DP on a small number of pixels in an image {\em guarantees} robustness of predictions against adversarial examples that can change up to that number of pixels.
The connection can be expanded to standard attack norms, including $\l_1$, $\l_2$, and $l_\infty$ norms.

\S\ref{sec:design} describes {\em \pixeldp}, the first certified defense against norm-bounded adversarial examples based on differential privacy (our second contribution).
Incorporating DP into the learning procedure to increase robustness to adversarial examples requires is completely different and orthogonal to using DP to preserve the privacy of the training set, the focus of prior DP ML literature~\cite{DBLP:conf/kdd/McSherryM09,2016arXiv160700133A,Chaudhuri:2011:DPE:1953048.2021036} (as \S~\ref{sec:related-work} explains).
A \pixeldp DNN includes in its architecture a {\em DP noise layer} that randomizes the network's computation, to enforce DP bounds on how much the distribution over its predictions can change with small, norm-bounded changes in the input.
At inference time, we leverage these DP bounds to implement a certified robustness check for individual predictions.
Passing the check for a given input {\em guarantees} that no perturbation exists up to a particular size that causes the network to change its prediction.
The robustness certificate can be used to either act exclusively on robust predictions, or to lower-bound the network's accuracy under attack on a test set.

\S\ref{sec:evaluation} presents the first experimental evaluation of a certified adversarial-examples defense for the Inception network trained on the ImageNet dataset (our third contribution).
We additionally evaluate \pixeldp on various network architectures for four other datasets (CIFAR-10, CIFAR-100, SVHN, MNIST), on which previous defenses -- both best effort and certified -- are usually evaluated.
Our results indicate that \pixeldp is (1) as effective at defending against attacks as today's state-of-the-art, best-effort defense~\cite{madry} and (2) more scalable and broadly applicable than a prior certified defense.

Our experience points to DP as a uniquely generic, broadly applicable, and flexible foundation for certified defense against norm-bounded adversarial examples (\S\ref{sec:analysis}, \S\ref{sec:related-work}).
We credit these properties to the {\em post-processing property of DP}, which lets us incorporate the certified defense in a network-agnostic way.

\section{DP-Robustness Connection}
\label{sec:dp-robustness-connection}

\subsection{Adversarial ML Background}
\label{sec:ml-background}

An ML model can be viewed as a function mapping inputs -- typically a vector of
numerical feature values -- to an output (a label for multiclass classification and a real number for regression).
Focusing on multiclass classification, we define a
{\em model} as a function $f \colon \mathbb{R}^n \rightarrow \mathcal{K}$ that
maps $n$-dimensional inputs to a label in the set $\mathcal{K} = \{1, \ldots,
K\}$ of all possible labels.
Such models typically map an input $x$ to a vector of scores $y(x) = (y_1(x),
\dotsc, y_K(x))$, such that $y_k(x) \in [0,1]$ and $\sum_{k=1}^K y_k(x) = 1$.
These scores are interpreted as a probability distribution over the labels, and
the model returns the label with highest probability, i.e., $f(x) = \arg\max_{k \in
\mathcal{K}} y_k(x)$.
We denote the function that maps input $x$ to $y$ as $Q$ and call it the {\em scoring function}; we denote the function that gives the ultimate prediction for input $x$ as $f$ and call it the {\em prediction procedure}.

\heading{Adversarial Examples.}
Adversarial examples are a class of attack against ML models, studied particularly on deep neural networks for multiclass image classification.
The attacker constructs a small change to a given, fixed input, that wildly changes the predicted output.
Notationally, if the input is $x$, we denote an adversarial version of that input by $x+\alpha$, where $\alpha$ is the change or perturbation introduced by the attacker.
When $x$ is a vector of pixels (for images), then $x_i$ is the $i$'th pixel in the image and $\alpha_i$ is the change to the $i$'th pixel.

It is natural to constrain the amount of change an attacker is allowed to make to the input, and often this is measured by the $p$-norm of the change, denoted by $\|\alpha\|_p$.
For $1 \leq p < \infty$, the $p$-norm of $\alpha$ is defined by $\|\alpha\|_p = (\sum_{i=1}^n |\alpha_i|^p )^{1/p}$; for $p=\infty$, it is $\|\alpha\|_\infty = \max_i |\alpha_i|$. Also commonly used is the $0$-norm (which is technically not a norm): $\|\alpha\|_0 = |\{ i : \alpha_i \neq 0 \}|$.
A small $0$-norm attack is permitted to arbitrarily change a few entries of the
input; for example, an attack on the image recognition system for self-driving
cars based on putting a sticker in the field of vision is such an
attack~\cite{physical-world-attacks-dawnsong}.
Small $p$-norm attacks for larger values of $p$ (including $p=\infty$) require the changes to the pixels to be small in an aggregate sense, but the changes may be spread out over many or all features.
A change in the lighting condition of an image may correspond to such an attack~\cite{physical-world-attacks-goodfellow, pei2017deepxplore}.
The latter attacks are generally considered more powerful, as they can easily remain invisible to human observers.
Other attacks that are not amenable to norm bounding exist~\cite{DBLP:journals/corr/abs-1801-02612, DBLP:journals/corr/abs-1805-07894, DBLP:journals/corr/abs-1801-02610}, but this paper deals exclusively with norm-bounded attacks. 

Let $B_p(r) := \{ \alpha \in \mathbb{R}^n : \|\alpha\|_p \leq r \}$ be the $p$-norm ball of radius $r$.
For a given classification model, $f$, and a fixed input, $x \in \mathbb{R}^n$, an
attacker is able to craft a successful adversarial example of size $L$ for a given
$p$-norm if they find $\alpha \in B_p(L)$ such that
$f(x + \alpha) \neq f(x)$.
The attacker thus tries to find a small change to $x$ that will change the predicted label.

\heading{Robustness Definition.}
Intuitively, a predictive model may be regarded as {\em robust} to adversarial examples if its output is {\em insensitive} to small changes to any {\em plausible} input that may be encountered in deployment.
To formalize this notion, we must first establish what qualifies as a {\em plausible} input.  This is difficult: the adversarial examples literature has not settled on such a definition. Instead, model robustness is typically assessed on inputs from a test set that are not used in model training -- similar to how accuracy is assessed on a test set and not a property on all plausible inputs.
We adopt this view of robustness.

Next, given an input, we must establish a definition for {\em insensitivity} to small changes to the input. We say a model $f$ is insensitive, or \emph{robust}, to attacks of $p$-norm $L$ on a given input $x$ if $f(x) = f(x + \alpha)$ for all $\alpha \in B_p(L)$.
If $f$ is a multiclass classification model based on label scores (as in \S\ref{sec:ml-background}), this is equivalent to:
\begin{multline} \label{eq:robustness}
  \forall \alpha \in B_p(L) \centerdot y_{k}(x + \alpha) > \max_{i : i \neq k} y_i(x + \alpha) ,
\end{multline}
where $k := f(x)$.
A small change in the input does not alter the scores so much as to change the predicted label.

\subsection{DP Background}
\label{sec:dp-background}

DP is concerned with whether the output of a computation over a database can reveal information about individual records in the database. To prevent such information leakage, randomness is introduced into the computation to hide details of individual records.

A randomized algorithm $A$ that takes as input a database $d$ and outputs a value in a space $O$ is said to satisfy \emph{$(\epsilon,\delta)$-DP with respect to a metric $\rho$ over databases} if, for any databases $d$ and $d'$ with $\rho(d,d') \leq 1$, and for any subset of possible outputs $S \subseteq O$, we have
\begin{equation} \label{eq:dp}
  P(A(d) \in S) \leq e^\epsilon P(A(d') \in S) + \delta .
\end{equation}
Here, $\epsilon>0$ and $\delta \in [0,1]$ are parameters that quantify the strength of the privacy guarantee.
In the standard DP definition, the metric $\rho$ is the Hamming metric, which simply counts the number of entries that differ in the two databases.
For small $\epsilon$ and $\delta$, the standard $(\epsilon,\delta)$-DP guarantee implies that changing a single entry in the database cannot change the output distribution very much.
DP also applies to general metrics $\rho$~\cite{chatzikokolakis2013broadening}, including $p$-norms relevant to norm-based adversarial examples.

Our approach relies on two key properties of DP.  First is the well-known {\em post-processing property}: any computation applied to the output of an $(\epsilon,\delta)$-DP algorithm remains $(\epsilon,\delta)$-DP.  Second is the {\em expected output stability property}, a rather obvious but not previously enunciated property that we prove in Lemma~\ref{lemma:expectation-bound}: the expected value of an $(\epsilon,\delta)$-DP algorithm with bounded output is not sensitive to small changes in the input.

\begin{lemma} \label{lemma:expectation-bound}
	{\bf (Expected Output Stability Bound)}
	Suppose a randomized function $A$, with bounded output $A(x) \in [0,b], \ b \in \mathbb{R}^+$, satisfies $(\epsilon,\delta)$-DP. Then the expected value of its output meets the following property:
	\[
	\forall \alpha \in B_p(1) \centerdot
	\E(A(x)) \leq e^\epsilon \E(A(x + \alpha)) + b\delta.
	\]
	The expectation is taken over the randomness in $A$.
\end{lemma}
\begin{proof}
	Consider any $\alpha \in B_p(1)$, and let $x' := x + \alpha$.
	We write the expected output as:
	\begin{align*}
	\E(A(x)) = \int_{0}^{b} P(A(x) > t) dt .
	\end{align*}
	We next apply Equation~\eqref{eq:dp} from the $(\epsilon, \delta)$-DP property:
	\begin{align*}
    \E(A(x)) \leq \ & e^\epsilon \Big( \int_{0}^{b} P(A(x') > t) dt \Big) + \int_{0}^{b}\delta dt \\
               =  \ & e^\epsilon \E(A(x')) + \int_{0}^{b}\delta dt.
	\end{align*}
	Since $\delta$ is a constant, $\int_{0}^{b}\delta dt = b\delta$.
\end{proof}

\subsection{DP-Robustness Connection}
\label{sec:actual-dp-robustness-connection}

The intuition behind using DP to provide robustness to adversarial examples is to create a {\em DP scoring function} such that, given an input example, the predictions are DP with regards to the features of the input (e.g. the pixels of an image).
In this setting, we can derive stability bounds for the expected output of the DP function using Lemma~\ref{lemma:expectation-bound}.  The bounds, combined with Equation~\eqref{eq:robustness}, give a rigorous condition (or {\em certification}) for robustness to adversarial examples.

Formally, regard the feature values (e.g., pixels) of an input $x$ as the records in a database, and consider a randomized scoring function $A$ that, on input $x$, outputs scores $(y_1(x), \dotsc, y_K(x))$ (with $y_k(x) \in [0,1]$ and $\sum_{k=1}^K y_k(x) = 1$).
We say that $A$ is an \emph{($\epsilon,\delta)$-pixel-level differentially private} (or \emph{($\epsilon,\delta$)-\pixeldp}) function if it satisfies $(\epsilon,\delta)$-DP (for a given metric). This is formally equivalent to the standard definition of DP, but we use this terminology to emphasize the context in which we apply the definition, which is fundamentally different than the context in which DP is traditionally applied in ML (see \S\ref{sec:related-work} for distinction).

Lemma~\ref{lemma:expectation-bound} directly implies bounds on the expected outcome on an $(\epsilon,\delta)$-\pixeldp scoring function:
\begin{corollary}
Suppose a randomized function $A$ satisfies $(\epsilon,\delta)$-\pixeldp with respect to a $p$-norm metric, and where $A(x) = (y_1(x), \dotsc, y_K(x)), \ y_k(x) \in [0,1]$:
  \begin{equation} \label{eq:expectation-bound}
  \forall k, \forall \alpha \in B_p(1) \centerdot
  \E(y_k(x)) \leq e^\epsilon \E(y_k(x + \alpha)) + \delta.
  \end{equation}
\end{corollary}
\begin{proof}
  For any $k$ apply Lemma~\ref{lemma:expectation-bound} with $b=1$.
\end{proof}

\begin{figure*}[t]
	\centering
	\subfigure[\pixeldp DNN Architecture]{
		\includegraphics[width=0.62\textwidth]{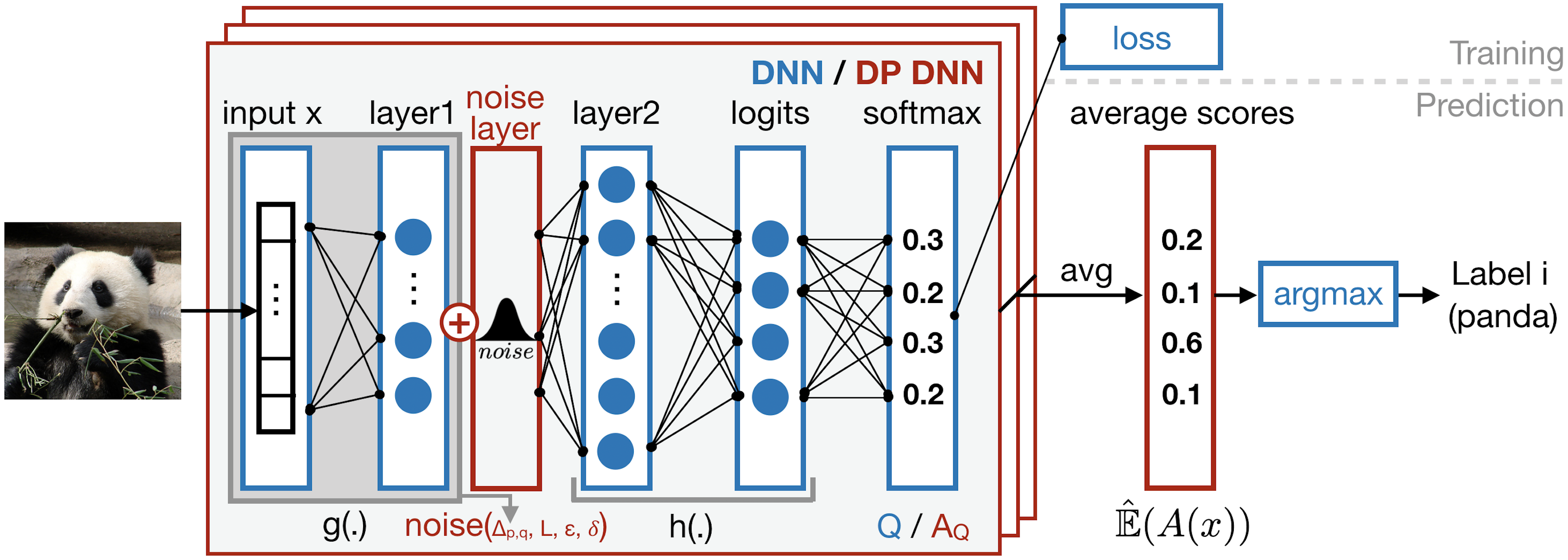}
		\label{fig:arch}
	}
	\hspace{10pt}
	% \hfill
	\subfigure[Robustness Test Example]{
		\includegraphics[width=0.21\textwidth]{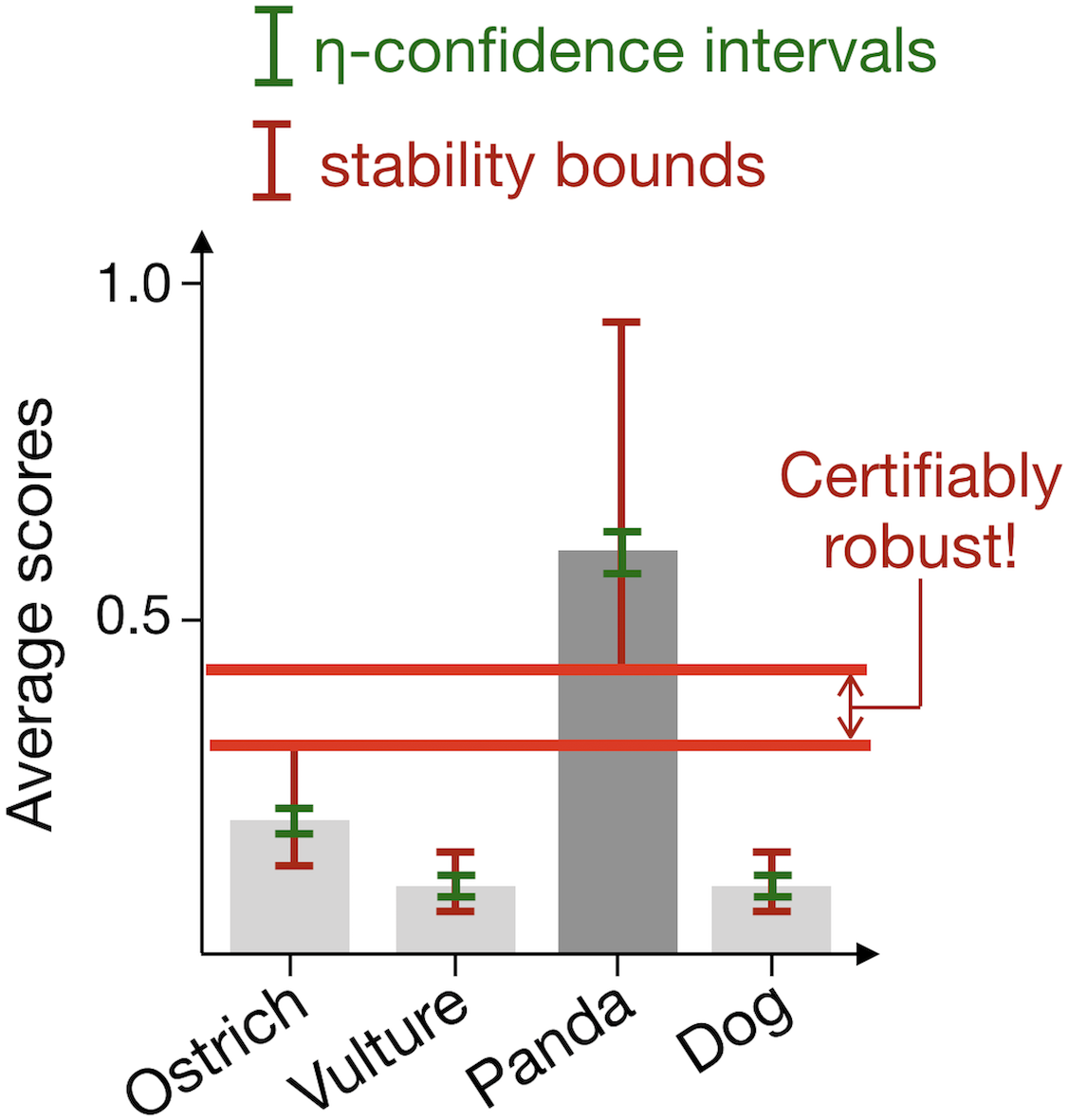}
		\label{fig:dp-bounds-ex}
	}
	\vspace{-5pt}
	\caption{
		\footnotesize
		\textbf{Architecture.}
		(a) In blue, the original DNN. In red, the noise layer that provides the
		$(\epsilon, \delta)$-DP guarantees. The noise can be added to the inputs
		{\em or} any of the following layers, but the distribution is rescaled
		by the sensitivity $\Delta_{p,q}$ of the computation performed by each layer before the
		noise layer.
		The DNN is trained with the original loss and optimizer (e.g., Momentum stochastic gradient descent).
		Predictions repeatedly call the $(\epsilon, \delta)$-DP DNN to measure its empirical expectation over the scores.
		(b) After adding the bounds for the measurement error between the empirical and true expectation (green) and the stability bounds from Lemma~\ref{lemma:expectation-bound} for a given attack size $L_{attack}$ (red), the prediction is certified robust to this attack size if the lower bound of the $\arg\max$ label does not overlap with the upper bound of any other labels.
	}
\vspace{-9pt}
\end{figure*}

Our approach is to transform a model's scoring function into a {\em randomized $(\epsilon, \delta)$-\pixeldp scoring function}, $A(x)$, and then have the model's prediction procedure, $f$, use A's {\em expected output over the DP noise}, $\E(A(x))$, as the label probability vector from which to pick the $\arg\max$.  I.e., $f(x) = \arg\max_{k \in \mathcal{K}} \E(A_k(x))$.
We prove that a model constructed this way allows the following robustness certification to adversarial examples:

\begin{proposition} \label{prop:robustness-condition}
{\bf (Robustness Condition)}
  Suppose $A$ satisfies $(\epsilon,\delta)$-\pixeldp with respect to a $p$-norm metric.  For any input $x$, if for some $k \in \mathcal{K}$,
  \begin{equation}\label{eq:robustness-condition}
    \E(A_k(x))
    > e^{2\epsilon} \max_{i : i \neq k} \E(A_i(x)) + (1+e^{\epsilon}) \delta ,
  \end{equation}
  then the multiclass classification model based on label probability vector $y(x)=(\E(A_1(x)),\dotsc,\E(A_K(x)))$ is robust to attacks $\alpha$ of size $\|\alpha\|_p \leq 1$ on input $x$.
\end{proposition}
\begin{proof}
  Consider any $\alpha \in B_p(1)$, and let $x' := x + \alpha$.
  From Equation~\eqref{eq:expectation-bound}, we have:
  \begin{align*}
    \E(A_k(x)) & \leq e^{\epsilon}  \E(A_k(x')) + \delta , & (a)\\
    \E(A_i(x')) & \leq e^{\epsilon} \E(A_i(x)) + \delta , \quad i \neq k . & (b)
  \end{align*}
  Equation (a) gives a lower-bound on $\E(A_k(x'))$; Equation (b) gives an upper-bound on $\max_{i \neq k} \E(A_i(x'))$.
  The hypothesis in the proposition statement (Equation~\eqref{eq:robustness-condition}) implies that the lower-bound of the expected score for label $k$ is strictly higher than the upper-bound for the expected score for any other label, which in turn implies the condition from Equation~\eqref{eq:robustness} for robustness at $x$.
  To spell it out:
  % $
  \begin{align*}
    \E(A_k(x'))
    \overset{\scalebox{.5}{Eq(a)}}{\geq} & \ \frac{\E(A_k(x))-\delta}{e^{\epsilon}} \\
    \overset{\scalebox{.5}{Eq(4)}}{>} & \ \frac{e^{2\epsilon} \max_{i : i \neq k} \E(A_i(x)) + (1+e^{\epsilon}) \delta - \delta}{e^\epsilon} \\
    = & \ e^{\epsilon} \max_{i : i \neq k} \E(A_i(x)) + \delta \\
    \overset{\scalebox{.5}{Eq(b)}}{\geq} & \max_{i : i \neq k} \E(A_i(x')) \\
    \implies \E(A_k(x')) > & \ \max_{i : i \neq k} \E(A_i(x + \alpha)) \ \forall \alpha \in B_p(1) ,
  \end{align*}
  the very definition of robustness at $x$ (Equation~\eqref{eq:robustness}).
\end{proof}

The preceding certification test is {\em exact} regardless of the value of the $\delta$ parameter of differential privacy: there is no failure probability in this test.
The test applies only to attacks of $p$-norm size of $1$, however all preceding results generalize to {\em attacks of $p$-norm size $L$}, i.e., when $\|\alpha\|_p \leq L$, by applying group privacy~\cite{dwork2014algorithmic}.
The next section shows how to apply group privacy (\S\ref{sec:noise-layer}) and generalize the certification test to make it practical (\S\ref{sec:prediction}).

\section{\pixeldp Certified Defense}
\label{sec:design}

\subsection{Architecture}
\label{sec:overview}

\pixeldp is a certified defense against $p$-norm bounded adversarial example attacks built on the preceding DP-robustness connection. \F\ref{fig:arch} shows an example PixelDP DNN architecture for multi-class image classification.
The original architecture is shown in blue; the changes introduced to make it PixelDP are shown in red.
Denote $Q$ the original DNN's scoring function; it is a deterministic map from images $x$ to a probability distribution over the $K$ labels $Q(x) = (y_1(x),\dotsc,y_K(x))$.
The vulnerability to adversarial examples stems from the unbounded sensitivity of $Q$ with respect to $p$-norm changes in the input.
Making the DNN $(\epsilon,\delta)$-\pixeldp involves adding {\em calibrated noise} to turn $Q$ into an $(\epsilon,\delta)$-DP randomized function $A_Q$; the expected output of that function will have bounded sensitivity to $p$-norm changes in the input.
We achieve this by introducing a \emph{noise layer} (shown in red in \F\ref{fig:arch}) that adds zero-mean noise to the output of the layer preceding it (layer1 in \F\ref{fig:arch}).
The noise is drawn from a Laplace or Gaussian distribution and its standard deviation is proportional to: (1) $L$, the $p$-norm attack bound for which we are constructing the network and (2) $\Delta$, the sensitivity of the pre-noise computation (the grey box in \F\ref{fig:arch}) with respect to $p$-norm input changes.

{\em Training} an $(\epsilon,\delta)$-\pixeldp network is similar to training the original network: we use the original loss and optimizer, such as stochastic gradient descent.
The major difference is that we alter the pre-noise computation to {\em constrain its sensitivity} with regards to $p$-norm input changes.
Denote $Q(x) = h(g(x))$, where $g$ is the pre-noise computation and $h$ is the subsequent computation that produces $Q(x)$ in the original network.
We leverage known techniques, reviewed in \S\ref{sec:training}, to transform $g$ into another function, $\tilde g$, that has a fixed sensitivity ($\Delta$) to $p$-norm input changes.
We then add the noise layer to the output of $\tilde g$, with a standard deviation scaled by $\Delta$ and $L$ to ensure $(\epsilon,\delta)$-\pixeldp for $p$-norm changes of size $L$.
Denote the resulting scoring function of the \pixeldp network: $A_{Q}(x) = h(\tilde g(x)+noise(\Delta, L, \epsilon, \delta))$, where $noise(.)$ is the function implementing the Laplace/Gaussian draw.
Assuming that the noise layer is placed such that $h$ only processes the DP output of $\tilde g(x)$ without accessing $x$ again (i.e., no skip layers exist from pre-noise to post-noise computation), the post-processing property of DP ensures that $A_{Q}(x)$ also satisfies $(\epsilon,\delta)$-\pixeldp for $p$-norm changes of size $L$.

{\em Prediction} on the $(\epsilon,\delta)$-\pixeldp scoring function, $A_{Q}(x)$, affords the robustness certification in Proposition~\ref{prop:robustness-condition} if the prediction procedure uses the {\em expected scores}, $\mathbb{E}(A_{Q}(x))$, to select the winning label for any input $x$.
Unfortunately, due to the potentially complex nature of the post-noise computation, $h$, we cannot compute this output expectation analytically.
We therefore resort to Monte Carlo methods to estimate it at prediction time and develop an approximate version of the robustness certification in Proposition~\ref{prop:robustness-condition} that uses standard techniques from probability theory to account for the estimation error (\S\ref{sec:prediction}).
Specifically, given input $x$, \pixeldp's prediction procedure invokes $A_{Q}(x)$ multiple times with new draws of the noise layer.
It then averages the results for each label, thereby computing an estimation $\hat\E(A_{Q}(x))$ of the expected score $\E(A_{Q}(x))$.
It then computes an $\eta$-confidence interval for $\hat\E(A_{Q}(x))$ that holds with probability $\eta$.
Finally, it integrates this confidence interval into the stability bound for the expectation of a DP computation (Lemma~\ref{lemma:expectation-bound}) to obtain $\eta$-confidence upper and lower bounds on the change an adversary can make to the average score of any label with a $p$-norm input change of size up to $L$.
\F\ref{fig:dp-bounds-ex} illustrates the upper and lower bounds applied to the average score of each label by the \pixeldp prediction procedure.
If the lower bound for the label with the top average score is strictly greater than the upper bound for every other label, then, with probability $\eta$, the \pixeldp network's prediction for input $x$ is {\em robust} to arbitrary attacks of $p$-norm size $L$.
The failure probability of this robustness certification, $1-\eta$, can be made arbitrarily small by increasing the number of invocations of $A_{Q}(x)$.

One can use \pixeldp's certification check in two ways: (1) one can decide only to actuate on predictions that are deemed robust to attacks of a particular size; or (2) one can compute, on a test set, a lower bound of a \pixeldp network's accuracy under $p$-norm bounded attack, independent of how the attack is implemented.
This bound, called {\em certified accuracy}, will hold no matter how effective future generations of the attack are.

The remainder of this section details the noise layer, training, and certified prediction procedures.
To simplify notation, we will henceforth use $A$ instead of $A_{Q}$.

\subsection{DP Noise Layer}
\label{sec:noise-layer}

The noise layer enforces $(\epsilon, \delta)$-\pixeldp by inserting noise inside the DNN using one of
two well-known DP mechanisms: the Laplacian and Gaussian mechanisms.
Both rely upon the {\em sensitivity} of the pre-noise layers (function $g$).
The sensitivity of a function $g$ is defined as the maximum change in output that can
be produced by a change in the input, given some distance metrics for the
input and output ($p$-norm and $q$-norm, respectively):
\[
  \Delta_{p,q} = \Delta_{p,q}^g = \max_{x,x':x \neq x'} \frac{\|g(x) - g(x')\|_q}{\|x - x'\|_p} .
\]

Assuming we can compute the sensitivity of the pre-noise layers (addressed shortly),
the noise layer leverages the Laplace and Gaussian mechanisms as follows.
On every invocation of the network on an input $x$ (whether for training or prediction)
the noise layer computes $g(x) + Z$, where the coordinates $Z = (Z_1,\dotsc,Z_m)$
are independent random variables from a noise distribution defined by the function $noise(\Delta, L, \epsilon, \delta)$.
\begin{itemize}
  \item Laplacian mechanism: $noise(\Delta, L, \epsilon, \delta)$ uses the Laplace distribution with mean zero and
  standard deviation $\sigma = \sqrt2\Delta_{p,1}L/\epsilon$; it gives $(\epsilon,0)$-DP.
  \item Gaussian mechanism: $noise(\Delta, L, \epsilon, \delta)$ uses the Gaussian distribution with mean zero and
  standard deviation $\sigma = \sqrt{2\ln(\frac{1.25}{\delta})} \Delta_{p,2}L/\epsilon$; it gives $(\epsilon,\delta)$-DP for $\epsilon \leq 1$.
\end{itemize}
Here, $L$ denotes the $p$-norm size of the attack against which the \pixeldp network provides $(\epsilon,\delta)$-DP; we call it the {\em \Lname}.
The noise formulas show that for a fixed noise standard deviation $\sigma$, the guarantee degrades gracefully: attacks twice as big halve the $\epsilon$ in the DP guarantee ($L \leftarrow 2L \Rightarrow \epsilon \leftarrow 2\epsilon$). This property is often referred as group privacy in the DP literature~\cite{dwork2014algorithmic}.

Computing the sensitivity of the pre-noise function $g$ depends on where we choose to place the noise layer in the DNN.
Because the post-processing property of DP carries the $(\epsilon, \delta)$-\pixeldp guarantee from the noise layer through the end of the network, a DNN designer has great flexibility in placing the noise layer anywhere in the DNN, as long as no skip connection exists from pre-noise to post-noise layers.
We discuss here several options for noise layer placement and how to compute sensitivity for each.
Our methods are not closely tied to particular network architectures and can therefore be applied on a wide variety of networks.

\heading{Option 1: Noise in the Image.}
The most straightforward placement of the noise layer is right after the input layer,
which is equivalent to adding noise to individual pixels of the image.
This case makes sensitivity analysis trivial: $g$ is the identity function, $\Delta_{1,1}=1$,
and $\Delta_{2,2}=1$.

\heading{Option 2: Noise after First Layer.}
Another option is to place the noise after the first hidden layer,
which is usually simple and standard for many DNNs.
For example, in image classification, networks often start with a convolution layer.
In other cases, DNNs start with fully connected layer.
These linear initial layers can be analyzed and their sensitivity computed as follows.

For linear layers, which consist of a linear operator with matrix form $W \in \mathbb{R}^{m,n}$,
the sensitivity is the matrix norm, defined as:
$\|W\|_{p,q} = \sup_{x : \|x\|_p \leq 1} \|Wx\|_q$.
Indeed, the definition and linearity of $W$ directly imply that $\frac{\|W x\|_q}{\|x\|_p} \leq \|W\|_{p,q}$, which means that:
  $\Delta_{p,q} = \|W\|_{p,q}$.
We use the following matrix norms~\cite{operator-norm}:
 $\|W\|_{1,1}$ is the maximum $1$-norm of $W$'s columns;
 $\|W\|_{1,2}$ is the maximum $2$-norm of $W$'s columns; and
 $\|W\|_{2,2}$ is the maximum singular value of $W$. For $\infty$-norm attacks,
 we need to bound $\|W\|_{\infty,1}$ or $\|W\|_{\infty,2}$, as our DP mechanisms
 require $q \in \{1,2\}$. However, tight bounds are computationally hard, so we
 currently use the following bounds: $\sqrt{n}\|W\|_{2,2}$ or
 $\sqrt{m}\|W\|_{\infty,\infty}$ where $\|W\|_{\infty,\infty}$ is the maximum
 $1$-norm of $W$'s rows.
 While these bounds are suboptimal and lead to results that are not as good as for
 $1$-norm or $2$-norm attacks, they allow us to include $\infty$-norm attacks
 in our frameworks.  We leave the study of better approximate bounds to future
 work.

For a convolution layer, which is linear but usually not expressed in matrix form, we reshape the input (e.g. the image)
as an $\mathbb{R}^{n d_{in}}$ vector, where
$n$ is the input size (e.g. number of pixels) and $d_{in}$ the number of input
channels (e.g. $3$ for the RGB channels of an image).
We write the convolution as an $\mathbb{R}^{n d_{out} \times n d_{in}}$ matrix where
each column has all filter maps corresponding to a given input channel, and zero values.
This way, a ``column'' of a convolution consists of all coefficients in the kernel that
correspond to a single input channel.
Reshaping the input does not change sensitivity.

\heading{Option 3: Noise Deeper in the Network.}
One can consider adding noise later in the network using the fact that when applying
two functions in a row $f_1(f_2(x))$ we have:
$\Delta^{(f_1 \circ f_2)}_{p,q} \leq \Delta^{(f_2)}_{p,r} \Delta^{(f_1)}_{r,q}$.
For instance, ReLU has a sensitivity of $1$ for $p, q \in \{1, 2, \infty\}$, hence a linear
layer followed by a ReLU has the same bound on the sensitivity as the linear layer alone.
However, we find that this approach for sensitivity analysis is difficult to generalize.
Combining bounds in this way leads to looser and looser approximations.
Moreover, layers such as batch normalization~\cite{ioffe2015batch}, which are popular in image
classification networks, do not appear amenable to such bounds (indeed, they are assumed away by some previous defenses~\cite{pmlr-v70-cisse17a}).
Thus, our general recommendation is to add the DP noise layer early in the network --
where bounding the sensitivity is easy -- and taking advantage of DP's
post-processing property to carry the sensitivity bound through the end of the
network.

\heading{Option 4: Noise in Auto-encoder.}
Pushing this reasoning further, we uncover an interesting placement possibility
that underscores the broad applicability and flexibility of our approach: adding
noise ``before'' the DNN in a {\em separately trained auto-encoder}.  An
auto-encoder is a special form of DNN trained to predict its own input,
essentially learning the identity function $f(x) = x$.
Auto-encoders are typically used to de-noise inputs
\cite{Vincent:2010:SDA:1756006.1953039}, and are thus a good fit for \pixeldp.
Given an image dataset, we can train a $(\epsilon,\delta)$-\pixeldp auto-encoder using the previous noise layer options.  We stack it before the predictive DNN doing the classification and
fine-tune the predictive DNN by running a few training steps on the combined auto-encoder and DNN.
Thanks to the decidedly useful post-processing property of DP, the stacked DNN and
auto-encoder are $(\epsilon,\delta)$-PixelDP.

This approach has two advantages. First, the auto-encoder can be
developed independently of the DNN, separating the concerns of learning a good
\pixeldp model and a good predictive DNN.
Second, \pixeldp auto-encoders are much smaller than predictive DNNs, and are thus
much faster to train.
We leverage this property to train the first certified model
for the large ImageNet dataset, using an auto-encoder and the
{\em pre-trained} Inception-v3 model, a substantial relief in terms of experimental work
 (\S\ref{sec:evaluation-methodology}).

\subsection{Training Procedure}
\label{sec:training}

The soundness of \pixeldp's certifications rely only on enforcing DP at prediction time.
Theoretically, one could remove the noise layer during training.
However, doing so results in near-zero certified accuracy in our experience.
Unfortunately, training with noise anywhere except in the image itself raises a new challenge: left unchecked the training procedure will scale up the sensitivity of the pre-noise layers, voiding the DP guarantees.

To avoid this, we alter the pre-noise computation to keep its sensitivity constant (e.g. $\Delta_{p,q} \leq 1$) during training.
The specific technique we use depends on the type of sensitivity we need to bound, i.e. on the values of $p$ and $q$.
For $\Delta_{1,1}$, $\Delta_{1,2}$, or $\Delta_{\infty,\infty}$, we normalize
the columns, or rows, of linear layers and use the regular optimization process
with fixed noise variance.  For $\Delta_{2,2}$, we run the projection step
described in \cite{pmlr-v70-cisse17a} after each gradient step from the stochastic gradient descent (SGD).
This makes the pre-noise layers Parseval tight frames, enforcing $\Delta_{2,2} =
1$.  For the pre-noise layers, we thus alternate between an SGD step with fixed noise variance
and a projection step.
Subsequent layers from the original DNN are left unchanged.

It is important to note that during training, we optimize for {\em a single draw of noise} to predict the true label for a training example $x$.
We estimate $\mathbb{E}(A(x))$ using multiple draws of noise only at prediction time.
We can interpret this as pushing the DNN to increase the margin between the expected score for the true label versus others. Recall from Equation~\eqref{eq:robustness-condition} that the bounds on predicted outputs give robustness only when the true label has a large enough margin compared to other labels. By pushing the DNN to give high scores to the true label $k$ at points around $x$ likely under the noise distribution, we increase $\mathbb{E}(A_{k}(x))$ and decrease $\mathbb{E}(A_{i \neq k}(x))$.

\subsection{Certified Prediction Procedure}
\label{sec:prediction}

For a given input $x$, the prediction procedure in a traditional DNN chooses the $\arg\max$ label based on the score vector obtained from a single execution of the DNN's deterministic scoring function, $Q(x)$.
In a \pixeldp network, the prediction procedure differs in two ways.
First, it chooses the $\arg\max$ label based on a Monte Carlo estimation of the expected value of the randomized DNN's scoring function, $\mathbb{\hat E}(A(x))$.  This estimation is obtained by invoking $A(x)$ multiple times with independent draws in the noise layer.
Denote $a_{k,n}(x)$ the $n^{th}$ draw from the distribution of the randomized function $A$ on the $k^{th}$ label, given $x$ (so $a_{k,n}(x) \sim A_{k}(x)$).
In Lemma~\ref{lemma:expectation-bound} we replace $\mathbb{E}(A_k(x))$ with $\mathbb{\hat E}(A_k(x))=\frac{1}{n} \sum_{n} a_{k,n}(x)$, where $n$ is the number of invocations of $A(x)$.
We compute $\eta$-confidence error bounds to account for the estimation error in our robustness bounds, treating each label's score as a random variable in $[0,1]$. We use Hoeffding's inequality \cite{hoeffding1963probability} or Empirical Bernstein bounds \cite{DBLP:conf/colt/MaurerP09} to bound the error in $\mathbb{\hat E}(A(x))$.
We then apply a union bound so that the bounds for each label are all valid together.
For instance, using Hoeffding's inequality, with probability $\eta$,~{\small $\mathbb{\hat E}^{lb}(A(x)) \triangleq \mathbb{\hat E}(A(x)) - \sqrt{\frac{1}{2n}ln(\frac{2k}{1-\eta})} \leq \E(A(x)) \leq \mathbb{\hat E}(A(x)) + \sqrt{\frac{1}{2n}ln(\frac{2k}{1-\eta})} \triangleq \mathbb{\hat E}^{ub}(A(x))$}.

Second, \pixeldp returns not only the prediction for $x$ ($\arg\max(\mathbb{\hat
E}(A(x)))$) but also a {\em robustness size certificate} for that prediction.
To compute the certificate, we extend
Proposition~\ref{prop:robustness-condition} to account for the measurement
error:
\begin{proposition} \label{prop:general-robustness-condition}
 {\bf (Generalized Robustness Condition)}
  Suppose $A$ satisfies $(\epsilon,\delta)$-PixelDP with respect to changes of
  size $L$ in $p$-norm metric.
  Using the notation from Proposition~\ref{prop:robustness-condition} further let $\mathbb{\hat E}^{ub}(A_i(x))$ and $\mathbb{\hat E}^{lb}(A_i(x))$ be the $\eta$-confidence upper and lower bound, respectively, for the Monte Carlo estimate $\mathbb{\hat E}(A_i(x))$. For any input $x$, if for some $k \in \mathcal{K}$,
  \[
  \mathbb{\hat E}^{lb}(A_k(x))
  > e^{2\epsilon} \max_{i : i \neq k} \mathbb{\hat E}^{ub}(A_i(x)) + (1+e^{\epsilon}) \delta ,
  \]
  then the multiclass classification model based on label probabilities $(\mathbb{\hat E}(A_1(x)),\dotsc,\mathbb{\hat E}(A_K(x)))$ is robust to attacks of $p$-norm $L$ on input $x$ with probability $\geq\eta$.
\end{proposition}
The proof is similar to the one for Proposition~\ref{prop:robustness-condition} and is detailed in Appendix~\ref{appendix:proposition-2-proof}.
Note that the DP bounds are not probabilistic even for $\delta>0$; the failure probability $1-\eta$ comes from the Monte Carlo estimate and can be made arbitrarily small with more invocations of $A(x)$.

Thus far, we have described \pixeldp certificates as binary with respect to a fixed attack bound, $L$: we either meet or do not meet a robustness check for $L$.
In fact, our formalism allows for a more nuanced certificate, which gives the {\em maximum attack size $L_{max}$} (measured in $p$-norm) against which the prediction on input $x$ is guaranteed to be robust: no attack within this size from $x$ will be able to change the highest probability.  $L_{max}$ can differ for different inputs.
We compute the robustness size certificate for input $x$ as follows.
Recall from \ref{sec:noise-layer} that the DP mechanisms have a noise standard deviation $\sigma$ that grows in $\frac{\Delta_{p,q}L}{\epsilon}$.
For a given $\sigma$ used at prediction time, we solve for the maximum $L$ for which the robustness condition in Proposition~\ref{prop:general-robustness-condition} checks out:

\begin{footnotesize}
	\begin{tcolorbox}[width=\linewidth, halign=left, colframe=grey, colback=white, boxsep=0mm, arc=2mm, left=8pt,right=4pt,top=4pt,bottom=4pt]
		$L_{max} = \max_{L \in \mathbb{R}^+} L$ such that\\
		$\mathbb{\hat E}^{lb}(A_{k}(x)) > e^{2\epsilon} \mathbb{\hat E}^{ub}(A_{i:i \neq k}(x)) + (1 + e^{\epsilon}) \delta$ AND either
		\begin{itemize}
			\item $\sigma = \Delta_{p,1}L/\epsilon$ and $\delta = 0$ (for Laplace) OR
			\item $\sigma = \sqrt{2\ln(1.25/\delta)} \Delta_{p,2}L/\epsilon$ and $\epsilon \leq 1$ (for Gaussian).
		\end{itemize}
	\end{tcolorbox}
\end{footnotesize}

The prediction on $x$ is robust to attacks up to $L_{max}$, so we award a robustness size certificate of $L_{max}$ for $x$.

We envision two ways of using robustness size certifications. First, when it
makes sense to only take actions on the subset of robust predictions (e.g., a
human can intervene for the rest), an application can use \pixeldp's certified
robustness on each prediction.
Second, when all points must be classified, \pixeldp gives a lower bound on the
accuracy under attack. Like in regular ML, the testing set is used as a proxy for
the accuracy on new examples. We can certify the minimum accuracy under attacks up
to a threshold size \T, that we call the {\em \Tname}. \T is an inference-time
parameter that can differ from the {\em \Lname} parameter, \L, that is used to
configure the standard deviation of the DP noise.
In this setting the certification is computed only on the testing set,
and is not required for each prediction. We only need the
highest probability label, which requires fewer noise draws.
\S\ref{sec:evaluation-performance} shows that in practice a few hundred draws are sufficient to retain a large fraction of the certified predictions, while a few dozen are needed for simple predictions.

\section{Evaluation}
\label{sec:evaluation}

We evaluate \pixeldp by answering four key questions:
\begin{enumerate}
	\item[{\bf Q1:}] How does DP noise affect model accuracy?
	\item[{\bf Q2:}] What accuracy can \pixeldp certify?
	\item[{\bf Q3:}] What is \pixeldp's accuracy under attack and how does it compare to that of other best-effort and certified defenses?
	\item[{\bf Q4:}] What is \pixeldp's computational overhead?
\end{enumerate}
We answer these questions by evaluating \pixeldp on five standard image
classification datasets and networks -- both large and small -- and comparing it
with one prior certified defense~\cite{wong2018provable} and one best-effort
defense~\cite{madry}.
\S\ref{sec:evaluation-methodology} describes the datasets, prior defenses, and
our evaluation methodology; subsequent sections address each question in turn.

{\em Evaluation highlights}: \pixeldp provides meaningful certified robustness
bounds for reasonable degradation in model accuracy on all datasets and
DNNs.  To the best of our knowledge, these include the first
certified bounds for large,
complex datasets/networks such as the Inception network on ImageNet and Residual
Networks on CIFAR-10. There, \pixeldp gives $60$\% certified accuracy for 2-norm
attacks up to $0.1$ at the cost of $8.5$ and $9.2$ percentage-point accuracy
degradation respectively.
Comparing \pixeldp to the prior certified defense on smaller datasets,
\pixeldp models give higher accuracy on clean examples (e.g., $92.9$\% vs.
$79.6$\% accuracy SVHN dataset), and higher robustness to $2$-norm attacks (e.g.,
$55$\% vs. $17$\% accuracy on SVHN for 2-norm attacks of $0.5$), thanks to the
ability to scale to larger models.
Comparing \pixeldp to the best-effort defense on larger models and datasets,
\pixeldp matches its accuracy (e.g., $87$\% for \pixeldp vs.
$87.3$\% on CIFAR-10) and robustness to $2$-norm bounded attacks.

\subsection{Methodology}
\label{sec:evaluation-methodology}

\begin{table*}
	\footnotesize
\begin{minipage}{0.6\textwidth}
    \centering
    \begin{tabular}{p{1.92cm}p{1.1cm}p{0.9cm}p{0.85cm}p{0.7cm}p{1.5cm}p{1.1cm}}
    \hline
    \textbf{Dataset} & \textbf{Image size} & \textbf{Training set size}
        & \textbf{Testing set size} & \textbf{Target labels}
        & \textbf{Classifier architecture} & \textbf{Baseline accuracy}\\
    \hline
    {\bf ImageNet}~\cite{imagenet-dataset} & 299x299x3 & 1.4M &50K  & 1000
        & Inception V3  & 77.5\% \\
    \hline
    {\bf CIFAR-100}~\cite{cifar-dataset} & 32x32x3 & 50K &10K  & 100 & ResNet & 78.6\% \\
    \hline
    {\bf CIFAR-10}~\cite{cifar-dataset} & 32x32x3 & 50K &10K  & 10 & ResNet & 95.5\% \\
    \hline
    {\bf SVHN}~\cite{svhn-dataset}  & 32x32x3 & 73K &26K  & 10 & ResNet  & 96.3\% \\
    \hline
    {\bf MNIST}~\cite{mnist-dataset} & 28x28x1 & 60K &10K  & 10 & CNN  & 99.2\% \\
    \hline
    \end{tabular}
\vspace{-3pt}
\caption{\textbf{Evaluation datasets and baseline models.}
 Last column shows the accuracy of the baseline, undefended models. The datasets are sorted based on descending order of scale or complexity.}
\vspace{-7pt}
\label{table:datasets-models}
\end{minipage}
\hfill
\begin{minipage}{0.37\textwidth}
  \centering
\begin{tabular}{cccc}
\hline
  \textbf{$p$-norm} & \textbf{DP}        & \textbf{Noise}    & \textbf{Sensitivity} \\
  \textbf{used}     & \textbf{mechanism} & \textbf{location} & \textbf{approach}    \\
\hline
$1$-norm & Laplace   & $1^{st}$ conv. & $\Delta_{1,1}=1$ \\
\hline
$1$-norm & Gaussian  & $1^{st}$ conv. & $\Delta_{1,2}=1$ \\
\hline
$2$-norm & Gaussian  & $1^{st}$ conv. & $\Delta_{2,2} \leq 1$ \\
\hline
$1$-norm & Laplace   & Autoencoder    & $\Delta_{1,1}=1$ \\
\hline
$2$-norm & Gaussian  & Autoencoder    & $\Delta_{2,2} \leq 1$ \\
\hline
\end{tabular}
\vspace{-3pt}
\caption{\textbf{Noise layers in \pixeldp DNNs.} For each DNN,
we implement defenses for different attack bound norms and DP
mechanisms.}
\vspace{-7pt}
\label{table:pixeldp-models}
        \end{minipage}
\end{table*}

\heading{Datasets.}
We evaluate \pixeldp on image classification tasks from five pubic datasets
listed in Table~\ref{table:datasets-models}.
The datasets are listed in descending order of size and complexity for classification tasks.
MNIST~\cite{mnist-dataset} consists of greyscale handwritten digits and is the easiest to classify.
SVHN~\cite{svhn-dataset} contains small, real-world digit images cropped from Google Street View photos of house numbers.
CIFAR-10 and CIFAR-100~\cite{cifar-dataset}
consist of small color images that are each centered on one object of one of 10
or 100 classes, respectively.
ImageNet~\cite{imagenet-dataset} is a large, production-scale image dataset with over 1 million images spread across 1,000 classes.

\heading{Models: Baselines and PixelDP.}
We use existing DNN architectures to train a high-performing baseline model for each dataset.
Table~\ref{table:datasets-models} shows the accuracy of the baseline models.
We then make each of these networks \pixeldp with regards to $1$-norm and $2$-norm bounded attacks.
We also did rudimentary evaluation of $\infty$-norm bounded attacks, shown in Appendix~\ref{appendix:linf-attacks}. While the \pixeldp formalism can support $\infty$-norm attacks, our results show that tighter bounds are needed to achieve a practical defense.
We leave the development and evaluation of these bounds for future work.

Table~\ref{table:pixeldp-models} shows the \pixeldp configurations we used for the $1$-norm and $2$-norm defenses.
The code is available at \url{https://github.com/columbia/pixeldp}.
Since most of this section focuses on models with $2$-norm attack bounds, we detail only those configurations here.

{\em ImageNet:} We use as baseline a pre-trained version of Inception-v3~\cite{szegedy2016rethinking} available in
Tensorflow~\cite{inception-model}.
To make it \pixeldp, we use the autoencoder approach from
\S\ref{sec:noise-layer}, which does not require a full retraining of Inception
and was instrumental in our support of ImageNet.
The encoder has three convolutional layers and tied encoder/decoder weights.
The convolution kernels are $10\times10\times32$, $8\times8\times32$, and $5\times5\times64$, with stride $2$.
We make the autoencoder \pixeldp by adding the DP noise after the first convolution.
We then stack the baseline Inception-v3 on the \pixeldp autoencoder and fine-tune it for $20k$ steps, keeping the autoencoder weights constant.

{\em CIFAR-10}, {\em CIFAR-100}, {\em SVHN}:
We use the same baseline architecture, a state-of-the-art Residual Network (ResNet)~\cite{DBLP:journals/corr/ZagoruykoK16}.
Specifically we use the Tensorflow implementation of a 28-10 wide ResNet~\cite{tensorflow_wide_resnet_code}, with the default parameters.
To make it \pixeldp, we slightly alter the architecture to remove the image standardization step.
This step makes sensitivity input dependent, which is harder to deal with in \pixeldp.
Interestingly, removing this step also increases the baseline's own accuracy for all three datasets.
In this section, we therefore report the accuracy of the changed networks as baselines.

{\em MNIST}: We train a Convolutional Neural Network (CNN) with two $5\times5$
convolutions (stride $2$, $32$ and $64$ filters)
followed by a $1024$ nodes fully connected layer.

\heading{Evaluation Metrics.}
We use two accuracy metrics to evaluate \pixeldp models: {\em conventional accuracy} and {\em certified accuracy}.
Conventional accuracy (or simply accuracy) denotes the fraction of a testing set on which a model is correct; it is the standard accuracy metric used to evaluate any DNN, defended or not.
Certified accuracy denotes the fraction of the testing set on which a certified model's predictions are both {\em correct} and {\em certified robust} for a given \Tname; it has become a standard metric to evaluate models trained with {\em certified defenses}~\cite{wong2018provable,raghunathan2018certified,2018arXiv180510265D}.
We also use {\em precision on certified examples}, which measures the number of correct predictions exclusively on examples that are certified robust for a given \Tname.
Formally, the metrics are defined as follows:

\begin{enumerate}
\item {\em Conventional accuracy} $\frac{\sum_{i=1}^n isCorrect(x_i)}{n}$, where $n$ is the
testing set size and $isCorrect(x_i)$ denotes a function returning 1 if the prediction on
test sample $x_i$ returns the correct label, and 0 otherwise.

\item {\em Certified accuracy} \\ $\frac{\sum_{i=1}^n (isCorrect(x_i) \& robustSize(scores, \epsilon, \delta, L) \geq T)}{n}$,
where $robustSize(scores, \epsilon, \delta, L)$ returns the certified robustness size, which is
then compared to the \Tname \T.

\item {\em Precision on certified examples} $\frac{\sum_{i=1}^n (isCorrect(x_i) \& robustSize(p_i, \epsilon, \delta, L) \geq T))}{\sum_{i=1}^n robustSize(p_i, \epsilon, \delta, L) \geq T)}$.
\end{enumerate}

For $T=0$ all predictions are robust, so certified accuracy is equivalent to
conventional accuracy.
Each time we report $L$ or $T$, we use a $[0, 1]$ pixel range.

\heading{Attack Methodology.}
Certified accuracy -- as provided by \pixeldp and other certified defense --
constitutes a guaranteed lower-bound on accuracy under {\em any} norm-bounded attack.
However, the accuracy obtained in practice when faced with a specific attack can
be much better.
How much better depends on the attack, which we evaluate in two steps.
We first perform an attack on 1,000 randomly picked samples (as is customary in defense evaluation~\cite{madry}) from the testing set.
We then measure conventional accuracy on the attacked test examples.

For our evaluation, we use the state-of-the art attack from Carlini and
Wagner~\cite{carlini-attacks}, that we run for $9$ iterations of binary search,
$100$ gradient steps without early stopping (which we empirically validated to
be sufficient), and learning rate $0.01$.
We also adapt the attack to our specific defense following~\cite{obfuscated-gradients}:
since \pixeldp adds noise to the DNN, attacks based on optimization may fail due
to the high variance of gradients, which would not be a sign of the absence of
adversarial examples, but merely of them being harder to find.
We address this concern by averaging the gradients over $20$ noise draws at each
gradient step.
Appendix \S\ref{appendix:pgd-attack-details} contains more details about the
attack, including sanity checks and another attack we ran similar to the one
used in~\cite{madry}.

\heading{Prior Defenses for Comparison.}
We use two state-of-art defenses as comparisons.  First, we use the
empirical defense model provided by the Madry Lab for
CIFAR-10~\cite{madry-models}.
This model is developed in the context of $\infty$-norm attacks. It
uses an adversarial training strategy to approximately minimize the worst case error
under malicious samples~\cite{madry}.
While inspired by robust optmization theory, this methodology is best effort
(see \S\ref{sec:related-work}) and supports no formal notion of robustness for
individual predictions, as we do in \pixeldp.
However, the Madry model performs better under the latest attacks than other best-effort defenses (it is in fact the only one not yet broken)~\cite{obfuscated-gradients}, and represents a good comparison point.

Second, we compare with another approach for certified robustness
against $\infty$-norm attacks~\cite{wong2018provable}, based on
robust optimization. This method does not yet scale to the largest datasets
(e.g. ImageNet), or the more complex DNNs (e.g. ResNet, Inception) both for
computational reasons and because not all necessary layers are yet supported
(e.g. BatchNorm). We thus use their largest released model/dataset, namely a CNN
with two convolutions and a $100$ nodes fully connected layer for the SVHN
dataset, and compare their robustness guarantees with our own networks' robustness guarantees.
We call this SVHN CNN model {\em RobustOpt}.

\subsection{Impact of Noise (Q1)}
\label{sec:evaluation-accuracy}

\begin{table}
  \footnotesize
  \begin{tabular}{cccccc}
\hline
  \textbf{Dataset}     & \textbf{Baseline}  & \textbf{$L=0.03$} & \textbf{$L=0.1$} & \textbf{$L=0.3$} & \textbf{$L=1.0$}      \\
\hline
  \textbf{ImageNet}    & 77.5\%             &  --     & 68.3\% & 57.7\%  & 37.7\%   \\
\hline
  \textbf{CIFAR-10}    & 95.5\%             &  93.3\% & 87.0\% & 70.9\%  & 44.3\%   \\
\hline
  \textbf{CIFAR-100}   & 78.6\%             &  73.4\% & 62.4\% & 44.3\%  & 22.1\%   \\
\hline
\textbf{SVHN}          & 96.3\%             &  96.1\% & 93.1\% & 79.6\%  & 28.2\%   \\
\hline
  \textbf{MNIST}       & 99.2\%             &  99.1\% & 99.1\% & 98.2\%  & 11\%     \\
\hline
\end{tabular}
\vspace{-3pt}
\caption{\textbf{Impact of \pixeldp noise on conventional accuracy.} For each DNN, we show different levels
of construction attack size $L$. Conventional accuracy degrades with noise level.}
\vspace{-7pt}
\label{table:pixeldp-acc}
\end{table}

{\em Q1: How does DP noise affect the conventional accuracy of our models?}
To answer, for each dataset we train up to four $(1.0,0.05)$-\pixeldp DNN, for \Lname $\L \in \{0.03, 0.1, 0.3, 1\}$.
Higher values of $L$ correspond to robustness against larger attacks and larger noise standard deviation $\sigma$.

\begin{figure*}[t]
	\centering
	\subfigure[ImageNet Certified Accuracy]{
		\includegraphics[width=0.47\textwidth]{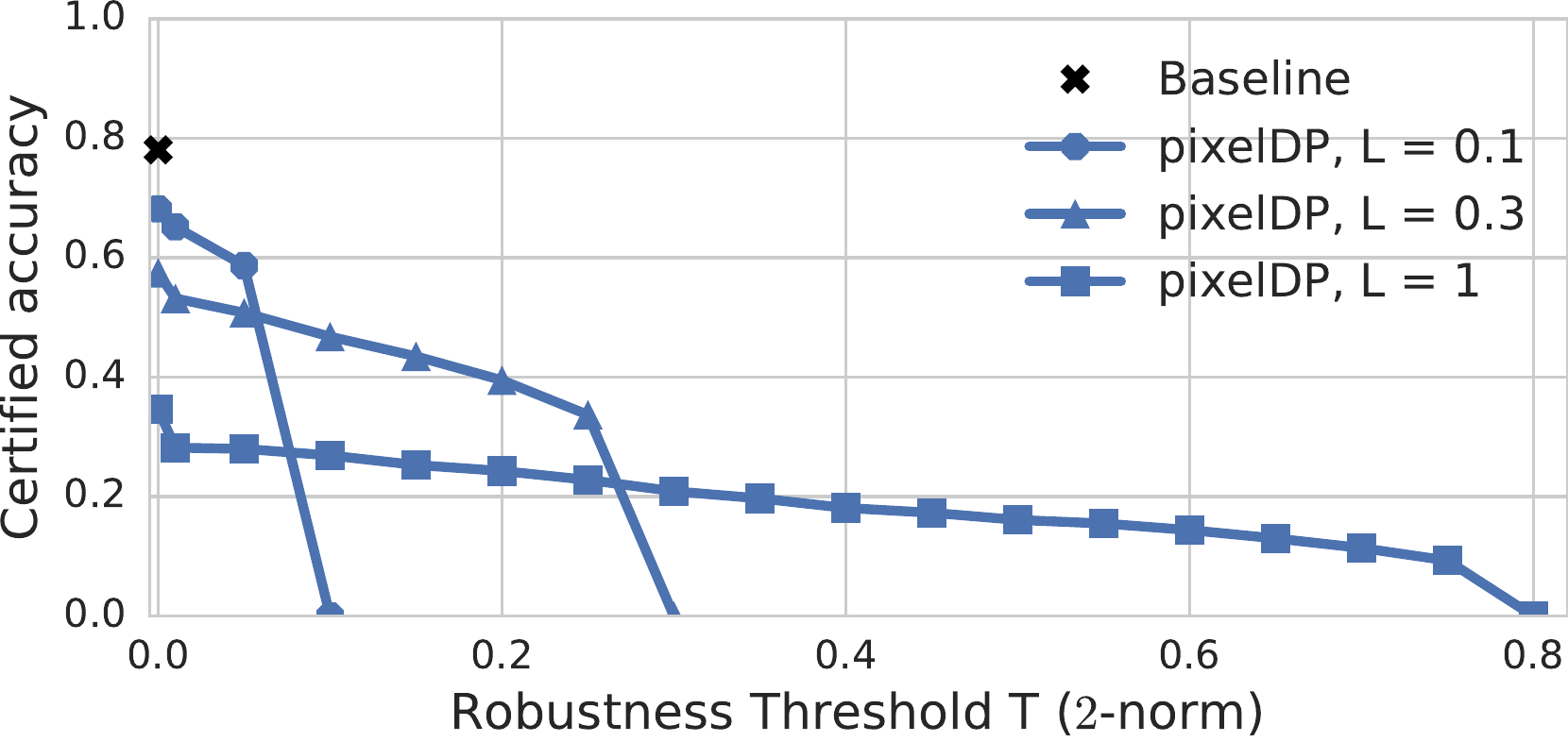}
		\label{sec:l2-robust-accuracy-imagenet}
	}
	\subfigure[CIFAR-10 Certified Accuracy]{
		\includegraphics[width=0.47\textwidth]{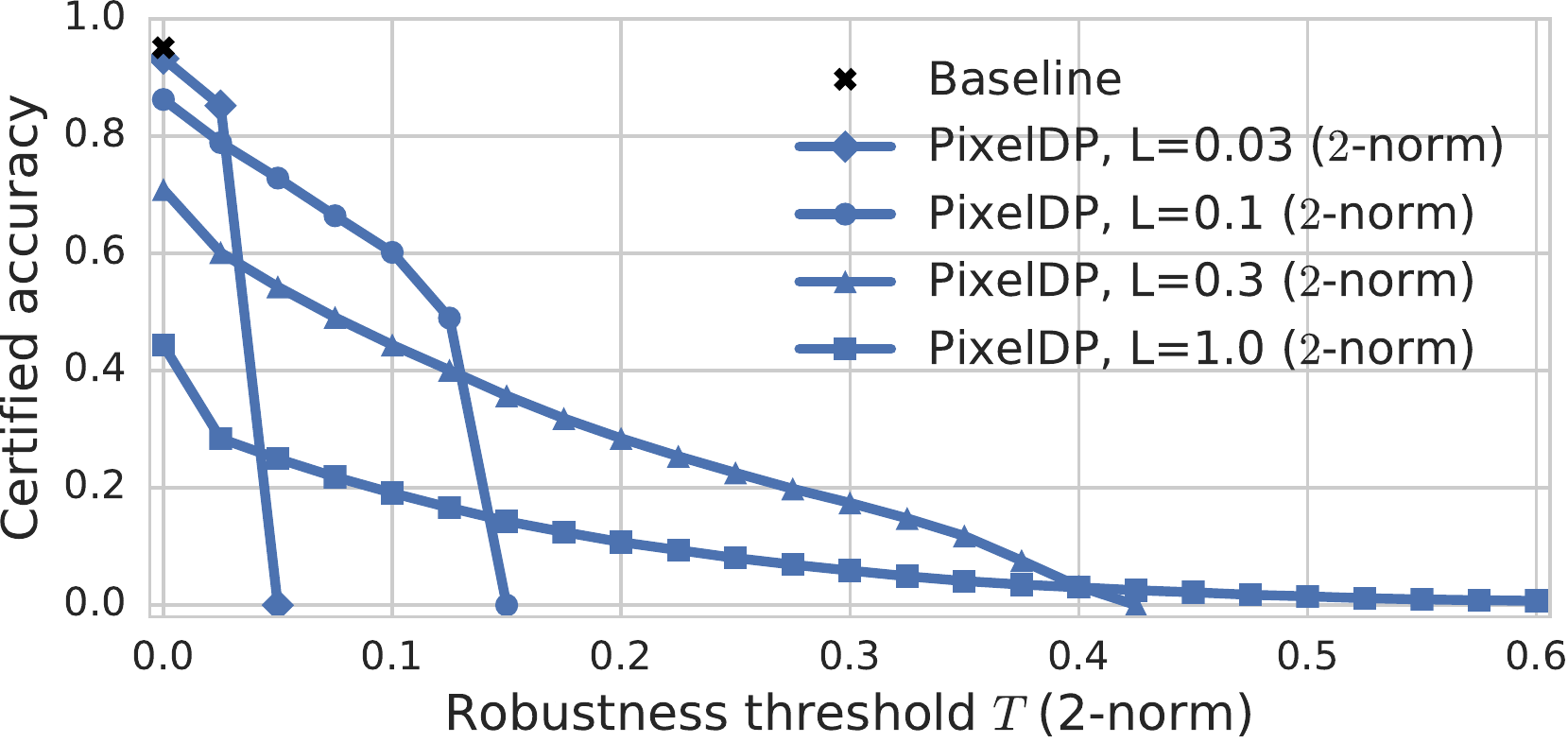}
		\label{sec:l2-robust-accuracy-cifar}
	}
	\vspace{-5pt}
	\caption{
		\textbf{Certified accuracy, varying the \Lname ($L$) and \Tname ($T$)}, on
		ImageNet auto-encoder/Inception and CIFAR-10 ResNet, 2-norm bounds.
		Robust accuracy at high Robustness thresholds (high $T$) increases with high-noise networks (high $L$).
		Low noise networks are both more accurate and more certifiably robust for low $T$.
	}
	\vspace{-7pt}
	\label{fig:robust_acc}
\end{figure*}

Table~\ref{table:pixeldp-acc} shows the conventional accuracy of these networks and highlights two parts of an answer to Q1.
First, at fairly low but meaningful \Lname (e.g., $L=0.1$), all of our DNNs exhibit reasonable accuracy loss -- {\em even on ImageNet}, a dataset on which no guarantees have been made to date!
ImageNet: The Inception-v3 model stacked on the \pixeldp auto-encoder has an accuracy of
68.3\% for $L=0.1$, which is reasonable degradation compared to the baseline of 77.5\% for the unprotected network.
CIFAR-10: Accuracy goes from 95.5\% without defense to 87\% with the $L=0.1$ defense.
For comparison, the Madry model has an accuracy of 87.3\% on CIFAR-10.
SVHN: our $L=0.1$ \pixeldp network achieves 93.1\% conventional accuracy, down from 96.3\% for the unprotected network.
For comparison, the $L=0.1$ RobustOpt network has an accuracy of 79.6\%, although they use a smaller DNN due to the computationally intensive method.

Second, as expected, constructing the network for larger attacks (higher $L$) progressively degrades accuracy.
ImageNet: Increasing $L$ to $0.3$ and then $1.0$ drops the accuracy to 57.7\%
and 37.7\%, respectively.
CIFAR-10: The ResNet with the least noise ($L=0.03$) reaches $93.3\%$ accuracy, close to
the baseline of $95.5\%$;
increasing noise levels ($L=(0.1, 0.3, 1.0)$) yields $87\%$,  $70.9\%$, and
$37.7\%$, respectively.
Yet, as shown in \S\ref{sec:evaluation-malicious-samples}, \pixeldp networks trained with fairly low $\L$ values (such as $\L=0.1$) already provide meaningful empirical protection against larger attacks.

\subsection{Certified Accuracy (Q2)}
\label{sec:evaluation-certification}

{\em Q2: What accuracy can \pixeldp certify on a test set?}
\F\ref{fig:robust_acc} shows the certified robust accuracy bounds for ImageNet and CIFAR-10 models, trained with various values of the \Lname $\L$.
The certified accuracy is shown as a function of the \Tname, $\T$.
We make two observations.
First, \pixeldp yields meaningful robust accuracy bounds even on large networks for ImageNet (see \F\ref{sec:l2-robust-accuracy-imagenet}), attesting the scalability of our approach.
The $\L=0.1$ network has a certified accuracy of 59\% for attacks smaller than $0.09$ in $2$-norm. The $\L=0.3$ network has a certified accuracy of 40\% to attacks up to size $0.2$.
To our knowledge, \pixeldp is the first defense to yield DNNs with certified bounds on accuracy under $2$-norm attacks on datasets of ImageNet's size and for large networks like Inception.

Second, \pixeldp networks constructed for larger attacks (higher $\L$, hence higher noise) tend to yield higher certified accuracy for high thresholds $T$.
For example, the ResNet on CIFAR-10 (see \F\ref{sec:l2-robust-accuracy-cifar}) constructed with $L=0.03$ has the highest robust accuracy
up to $T=0.03$, but the ResNet constructed with $L=0.1$ becomes better past that threshold.
Similarly, the $L=0.3$ ResNet has higher robust accuracy than the $L=0.1$ ResNet
above the $0.14$ $2$-norm \Tname.

We ran the same experiments on SVHN, CIFAR-100 and MNIST models but omit the
graphs for space reasons.
Our main conclusion -- that adding more noise (higher $L$)
hurts both conventional and low $T$ certified accuracy, but enhances the quality of
its high $T$ predictions -- holds in all cases.
Appendix \ref{appendix:design-choices} discusses the impact of some design choices on robust accuracy, and Appendix \ref{appendix:linf-attacks} discusses \pixeldp guarantees as compared with previous certified defenses for $\infty$-norm attacks. While \pixeldp does not yet yield strong $\infty$-norm bounds, it provides meaningful certified accuracy bounds for $2$-norm attacks, including on much larger and more complex datasets and networks than those supported by previous approaches.

\subsection{Accuracy Under Attack (Q3)}
\label{sec:evaluation-malicious-samples}

A standard method to evaluate the strength of a defense is to
measure the conventional accuracy of a defended model on malicious samples
obtained by running a state-of-the-art attack against samples in a held-out testing
set~\cite{madry}.
We apply this method to answer three aspects of question {\em Q3:
(1) Can \pixeldp help defend complex models on large datasets in practice?
(2) How does \pixeldp's accuracy under attack compare to state-of-the-art defenses?
(3) How does the accuracy under attack change for certified predictions?}

\begin{figure}[t]
	\centering
	\footnotesize
	\includegraphics[width=0.47\textwidth]{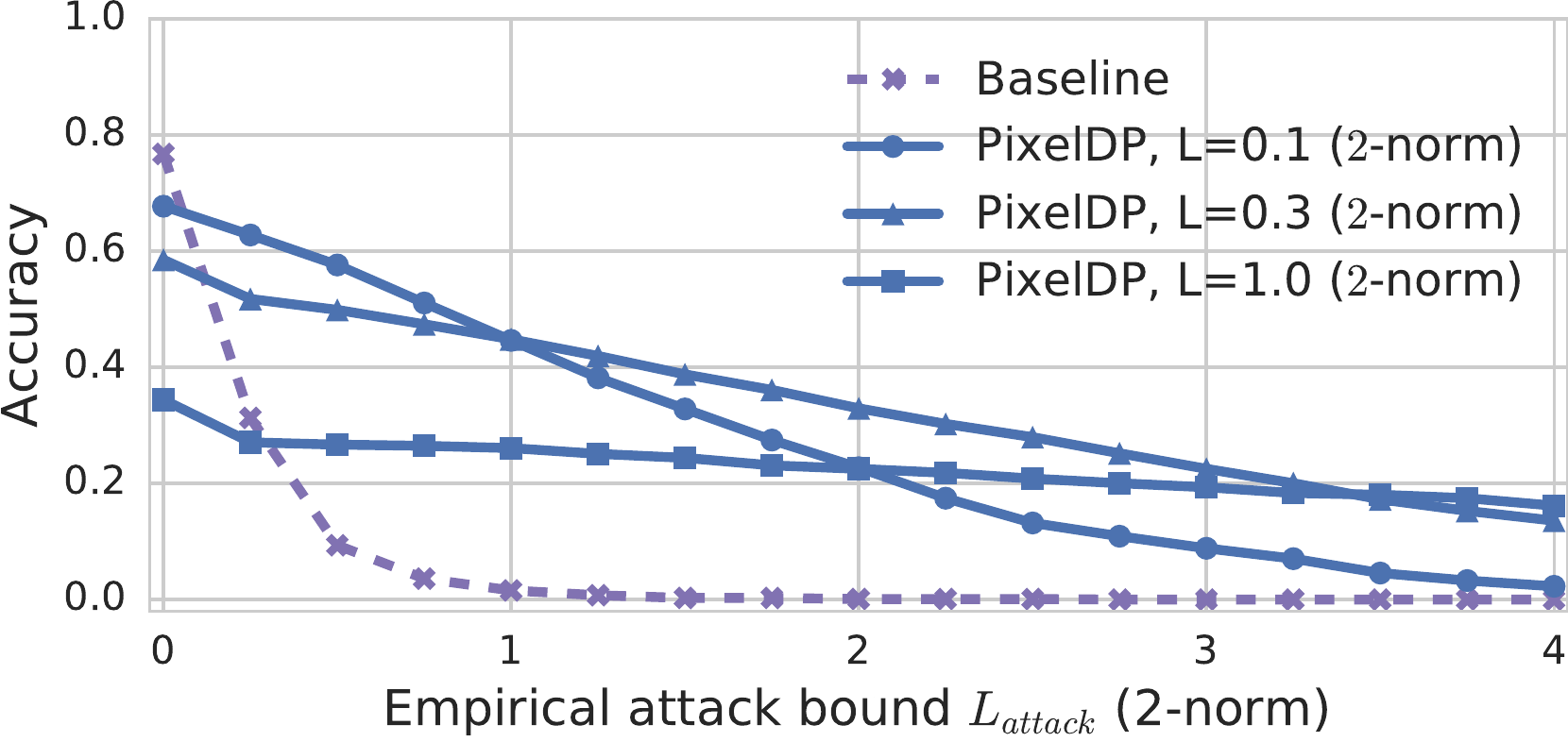}
	\vspace{-2pt}
	\caption{
		{\bf Accuracy under attack on ImageNet.} For the ImageNet auto-encoder plus Inception-v3, $L \in \{0.1,0.3,1.0\}$ $2$-norm attacks. The \pixeldp auto-encoder increases the robustness of Inception against $2$-norm attacks.
	}
	\vspace{-7pt}
	\label{fig:accuracy-under-attack-imagenet}
\end{figure}

\heading{Accuracy under Attack on ImageNet.}
We first study conventional accuracy under attack for \pixeldp models on ImageNet.
\F\ref{fig:accuracy-under-attack-imagenet} shows this metric for $2$-norm
attacks on the baseline Inception-v3 model, as well as three defended versions,
with a stacked \pixeldp auto-encoder trained with \Lname $L \in \{0.1, 0.3, 1.0\}$.
\pixeldp makes the model significantly more robust to attacks.
For attacks of size $L_{attack} = 0.5$, the baseline model's accuracy drops to
11\%, whereas the $L=0.1$ \pixeldp model's accuracy remains above 60\%. At $L_{attack} = 1.5$,
the baseline model has an accuracy of $0$, but the $L=0.1$ \pixeldp is still
at 30\%, while the $L=0.3$ \pixeldp model have more that 39\% accuracy.

\begin{figure*}[t]
  \centering
  \footnotesize
  \subfigure[CIFAR-10\vspace{-8pt}]{
    \includegraphics[width=0.47\textwidth]{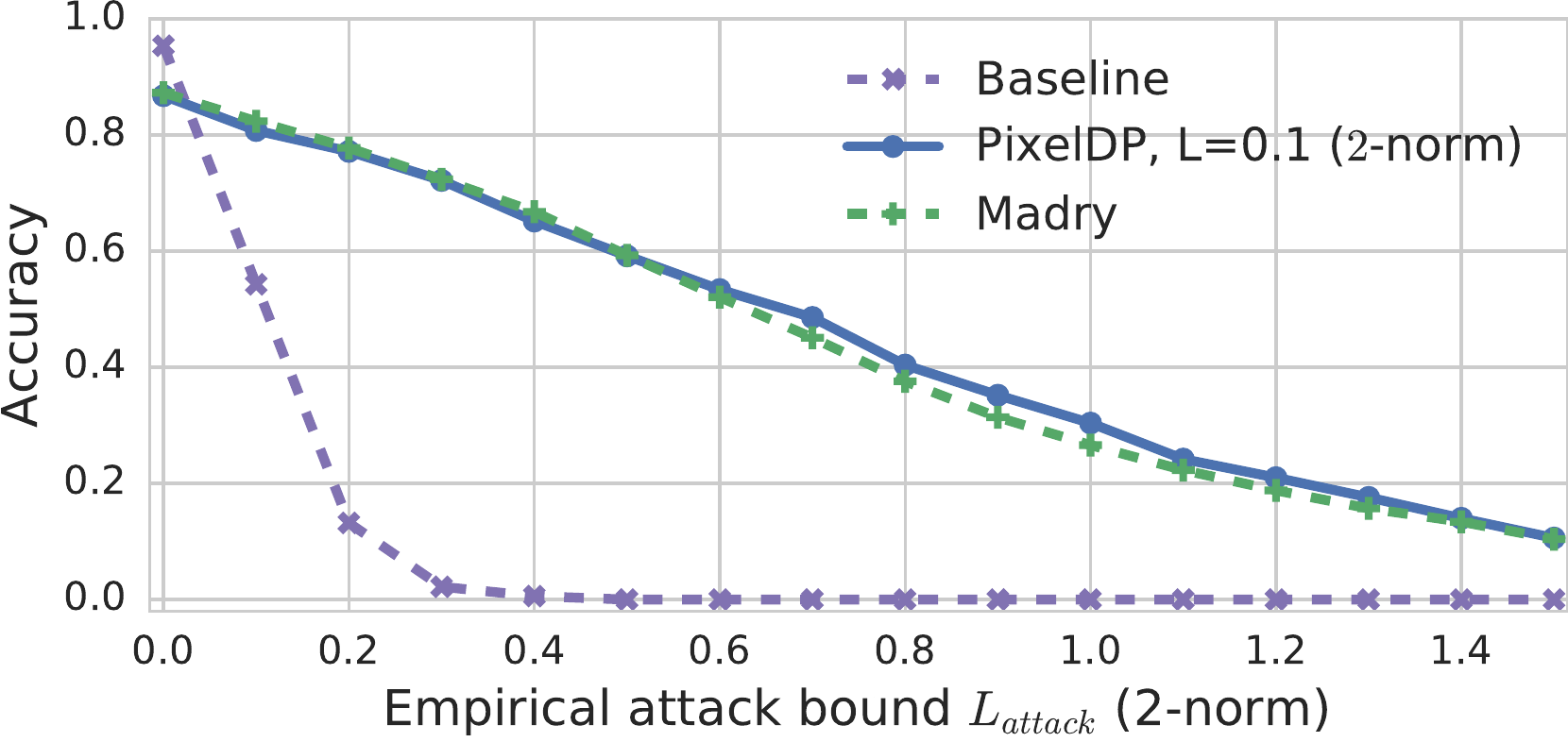}
    \label{fig:empirical_robustness_madry}
  }
  \subfigure[SVHN\vspace{-8pt}]{
    \includegraphics[width=0.47\textwidth]{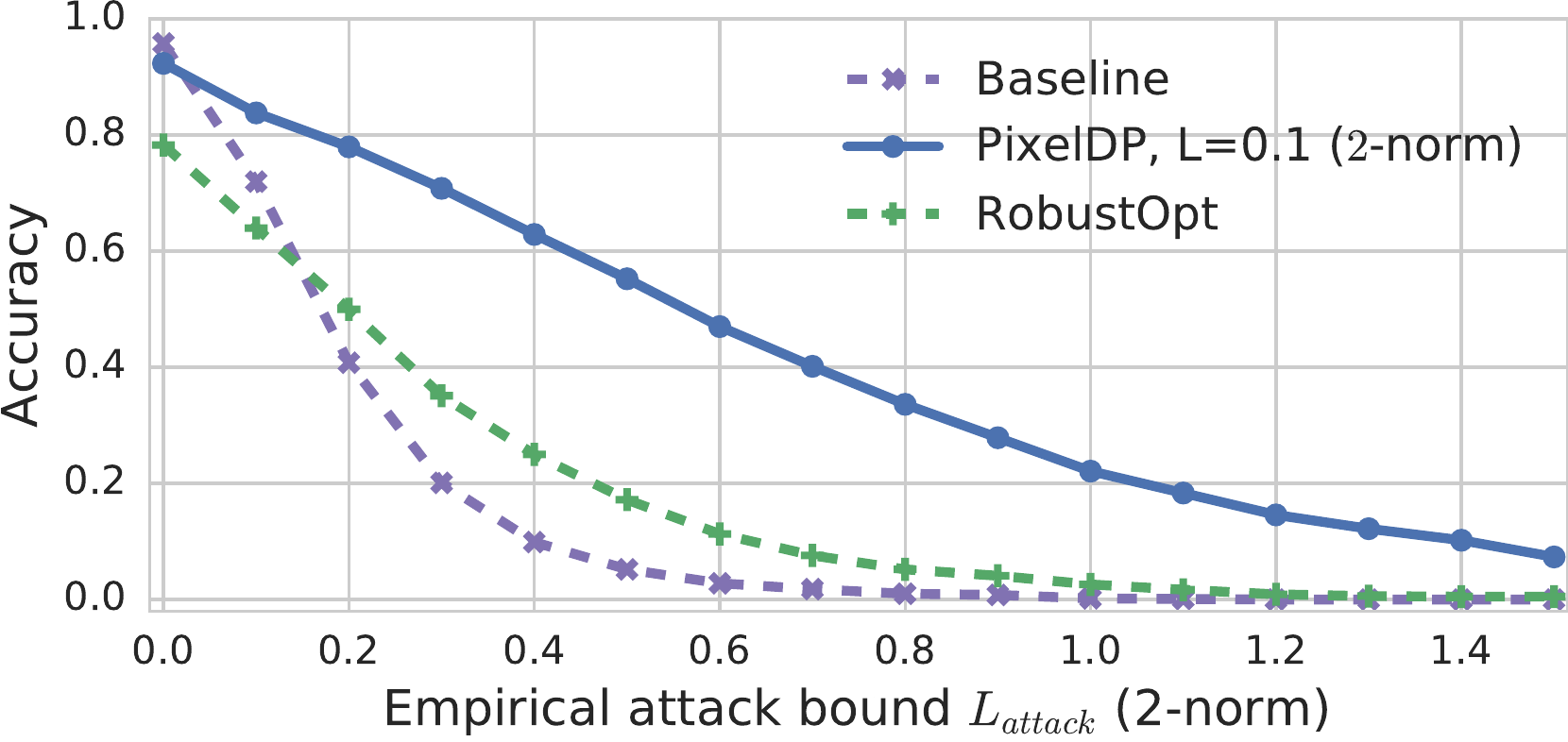}
    \label{fig:empirical_robustness_cmu}
  }
  \vspace{-5pt}
    \caption{
      {\bf Accuracy under $2$-norm attack for \pixeldp vs. Madry and RobustOpt}, CIFAR-10 and SVHN.
        For $2$-norm attacks, \pixeldp is on par with Madry until $L_{attack} \geq 1.2$; RobustOpt support only small models, and has lower accuracy.
    }
	\vspace{-7pt}
    \label{fig:empirical_robustness_madry_comparison}
\end{figure*}

\heading{Accuracy under Attack Compared to Madry.}
\F\ref{fig:empirical_robustness_madry} compares conventional accuracy of a
\pixeldp model to that of a Madry model on CIFAR-10, as the \Lattackname
increases for $2$-norm attacks.  For $2$-norm attacks, our model achieves
conventional accuracy on par with, or slightly higher than, that of the Madry
model.  Both models are dramatically more robust under this attack compared to
the baseline (undefended) model.
For $\infty$-norm attacks our model does not fare as well, which is expected as
the \pixeldp model is trained to defend against $2$-norm attacks, while the Madry model
is optimized for $\infty$-norm attacks. For $L_{attack}=0.01$,
\pixeldp's accuracy is 69\%, 8 percentage points lower than Madry's. The gap increases
until \pixeldp arrives at $0$ accuracy for $L_{attack}=0.06$, with Madry still having
22\%.
Appendix~\S\ref{appendix:linf-attacks} details this evaluation.

\heading{Accuracy under Attack Compared to RobustOpt.}
\F\ref{fig:empirical_robustness_cmu} shows a similar comparison with the RobustOpt defense \cite{wong2018provable},
which provides certified accuracy bounds for $\infty$-norm attacks. We use the SVHN dataset for the comparison as the RobustOpt defense has not yet been applied to larger datasets.
Due to our support of larger DNN (ResNet), \pixeldp starts with higher accuracy, which it maintains
under $2$-norm attacks. For attacks of $L_{attack} = 0.5$, RobustOpt is bellow $20$\% accuracy, and \pixeldp above $55$\%.
Under $\infty$-norm attacks, the behavior is different: \pixeldp has the advantage up to $L_{attack} = 0.015$ (58.8\% to 57.1\%), and RobustOpt is better thereafter. For instance, at $L_{attack} = 0.03$, \pixeldp has 22.8\% accuracy, to RobustOpt's 32.7\%.
Appendix \S\ref{appendix:linf-attacks} details the $\infty$-norm attack evaluation.

\begin{figure}[t]
    \centering
    \footnotesize
    \includegraphics[width=0.48\textwidth]{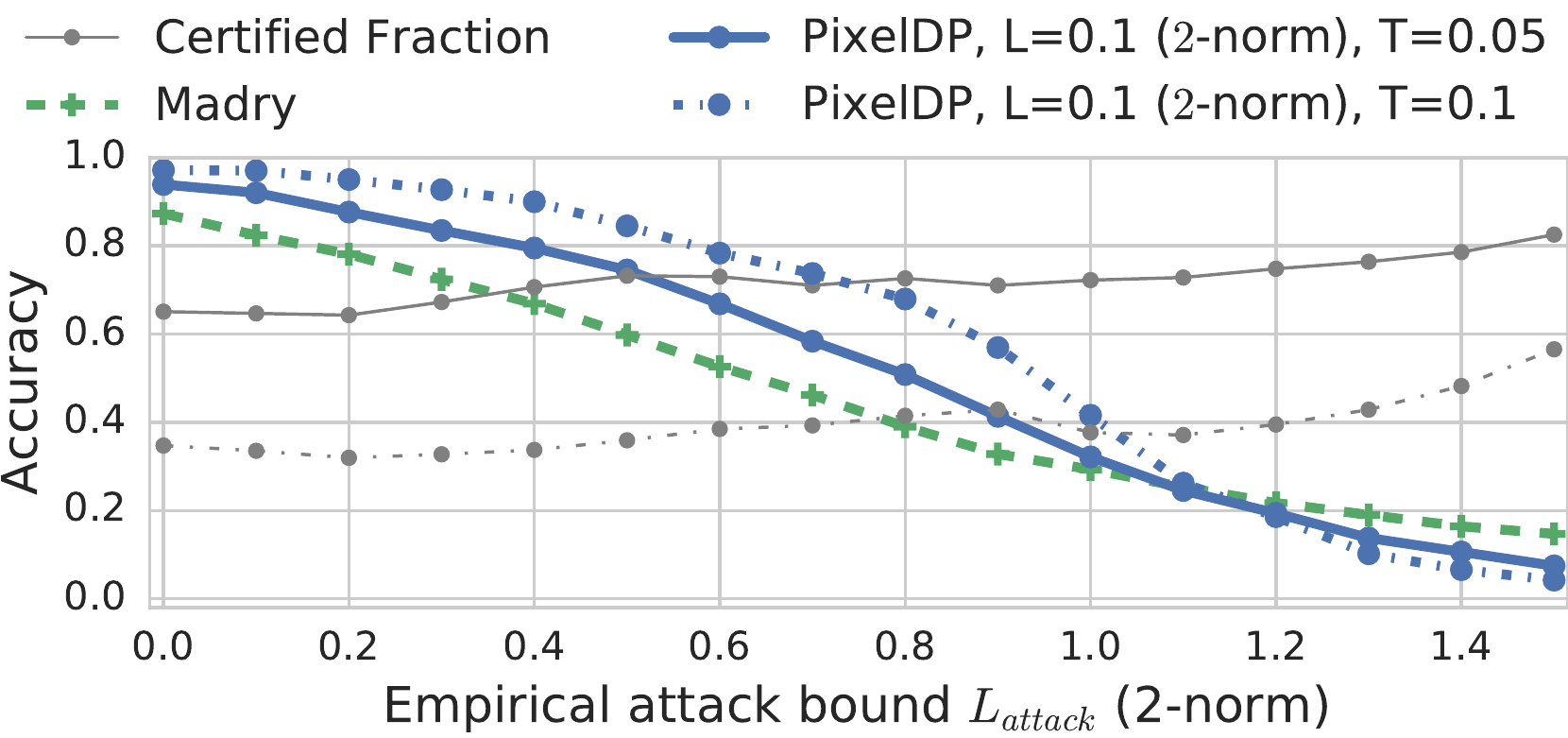}
    \vspace{-8pt}
    \caption{
      {\bf \pixeldp certified predictions vs. Madry accuracy, under attack},
      CIFAR-10 ResNets, $2$-norm attack.
        \pixeldp makes fewer but more correct predictions up to $L_{attack} = 1.0$.
    }
	\vspace{-7pt}
    \label{fig:robust_precision-under-attack-l2-attack}
\end{figure}

\heading{Precision on Certified Predictions Under Attack.}
Another interesting feature of \pixeldp is its ability to make certifiably robust
predictions. We compute the accuracy of these certified predictions under attack -- which we term {\em robust precision} --
and compare them to predictions of the Madry network that do not provide such a
certification.
\F\ref{fig:robust_precision-under-attack-l2-attack} shows the results of considering
only predictions with a certified robustness above $0.05$ and $0.1$.
It reflects the benefit to be gained by applications that can leverage our
theoretical guarantees to filter out non-robust predictions.
We observe that \pixeldp's robust predictions are {\em substantially more correct} than
Madry's predictions up to an \Lattackname of $1.1$.
For $T=0.05$ \pixeldp's robust predictions are $93.9$\% accurate, and up to 10
percentage points more correct under attack for $L_{attack} \leq 1.1$.
A robust prediction is given for above $60$\% of the data points.
The more conservative the robustness test is (higher $T$), the more correct
\pixeldp's predictions are, although it makes fewer of them (Certified Fraction lines).

Thus, for applications that can afford to not act on a minority of
the predictions, \pixeldp's robust predictions under 2-norm attack are substantially
more precise than Madry's.
For applications that need to act on every prediction, \pixeldp offers on-par
accuracy under 2-norm attack to Madry's.
Interestingly, although our defense is trained for $2$-norm attacks, the first conclusion still holds for $\infty$-norm attacks;
the second (as we saw) does not.

\subsection{Computational Overhead (Q4)}
\label{sec:evaluation-performance}

{\em Q4: What is \pixeldp's computational overhead?}
We evaluate overheads for training and prediction.  \pixeldp adds little overhead for {\em training}, 
as the only additions are a random noise tensor and sensitivity computations.
On our GPU, the CIFAR-10 ResNet baseline takes on average $0.65s$ per training step.
\pixeldp versions take at most $0.66s$ per training step (1.5\% overhead). This represents
a significant benefit over adversarial training (e.g. Madry) that requires finding
good adversarial attacks for each image in the mini-batch at each gradient step, and over
robust optimization (e.g. RobustOpt) that requires solving a constrained optimization problem
at each gradient step.
The low training overhead is instrumental to our support of large models and datasets.

\pixeldp impacts {\em prediction} more substantially, since it uses multiple
noise draws to estimate the label scores.
Making a prediction for a single image with $1$ noise draw takes $0.01s$ on
average.  Making $10$ draws brings it only to $0.02s$, but $100$ requires
$0.13s$, and $1000$, $1.23s$.
It is possible to use Hoeffding's inequality~\cite{hoeffding1963probability} to
bound the number of draws necessary to distinguish the highest score with
probability at least $\eta$, given the difference between the top two
scores $y_{max} - y_{second-max}$. %These bounds can be loose, and we do not know the
Empirically, we found that $300$ draws were typically necessary
to properly certify a prediction, implying a prediction time of $0.42s$ seconds, a
$42\times$ overhead.  This is parallelizable, but resource consumption is still
substantial.
To make simple predictions -- distinguish the top label when we must make a
prediction on all inputs -- 25 draws are enough in practice, reducing the overhead
to $3\times$.

\section{Analysis}
\label{sec:analysis}

We make three points about PixelDP's guarantees and applicability.
First, we emphasize that our Monte Carlo approximation of the function $x \mapsto \E(A(x))$ is \emph{not} intended to be a DP procedure. Hence, there is no need to apply composition rules from DP, because we do not need this randomized procedure to be DP. Rather, the Monte Carlo approximation $x \mapsto \hat\E(A(x))$ is just that: an approximation to a function $x \mapsto \E(A(x))$ whose robustness guarantees come from Lemma~\ref{lemma:expectation-bound}.
The function $x \mapsto \hat\E(A(x))$ does not satisfy DP, but because we can control the Monte Carlo estimation error using standard tools from probability theory, it is also robust to small changes in the input, just like $x \mapsto \E(A(x))$.

Second, Proposition~\ref{prop:robustness-condition} is not a high probability result; it is valid with probability $1$ even when $A$ is $(\epsilon, \delta>0)$-DP.
The $\delta$ parameter can be thought of as a ``failure probability'' of an $(\epsilon, \delta)$-DP mechanism: a chance that a small change in input will cause a big change in the probability of some of its outputs.
However, since we know that $A_k(x) \in [0,1]$, the worst-case impact of such failures on the expectation of the output of the $(\epsilon, \delta)$-DP mechanism is {\em at most $\delta$}, as proven in Lemma~\ref{lemma:expectation-bound}.
Proposition~\ref{prop:robustness-condition} explicitly accounts for this worst-case impact (term $(1+e^{\epsilon}) \delta$ in Equation~\eqref{eq:robustness-condition}).

Were we able to compute $\E(A(x))$ analytically, \pixeldp would output deterministic robustness certificates. In practice however, the exact value is too complex to compute, and hence we approximate it using a Monte Carlo method. This adds probabilistic measurement error bounds, making the final certification (Proposition~\ref{prop:general-robustness-condition}) a high probability result. However, the uncertainty comes exclusively from the Monte Carlo integration -- and can be made arbitrarily small with more runs of the \pixeldp DNN -- and not from the underlying $(\epsilon, \delta)$-DP mechanism $A$.
Making the uncertainty small gives an adversary a small chance to fool a PixelDP network into thinking that its prediction is robust when it is not.
The only ways an attacker can increase that chance is by either submitting the same attack payload many times or gaining control over PixelDP's source of randomness.

Third,
\pixeldp applies to any task for which we can measure changes to input in a meaningful $p$-norm, and bound the sensitivity to such changes at a given layer in the DNN (e.g. sensitivity to a bounded change in a word frequency vector, or a change of class for categorical attributes).
\pixeldp also applies to multiclass classification where the prediction procedure returns several top-scoring labels.
Finally, Lemma~\ref{lemma:expectation-bound} can be extended to apply to DP mechanism with (bounded) output that can also be negative, as shown in Appendix~\ref{appendix:regression-extension}. \pixeldp thus directly applies to DNNs for regression tasks (i.e. predicting a real value instead of a category) as long as the output is bounded (or unbounded if $\delta=0)$.
The output can be bounded due to the specific task, or by truncating the results to a large range of values and using a comparatively small $\delta$. 

\section{Related Work}
\label{sec:related-work}

Our work relates to a significant body of work in adversarial examples and beyond.
Our main contribution to this space is to introduce a new and very different direction for building {\em certified defenses}. Previous attempts have built on robust optimization theory. In \pixeldp we propose a new approach built on differential privacy theory which exhibits a level of flexibility, broad applicability, and scalability that exceeds what robust optimization-based certified defenses have demonstrated.
While the most promising way to defend against adversarial examples is still an open question, we observe undebatable benefits unique to our DP based approach,
such as the post-processing guarantee of our defense.
In particular, the ability to prepend a defense to unmodified networks via a \pixeldp auto-encoder, as we did to defend Inception with {\em no structural changes}, is unique among certified (and best-effort) defenses.

\heading{Best-effort Defenses.}
Defenders have used multiple heuristics to empirically increase DNNs' robustness.
These defenses include model distillation~\cite{papernot2016distillation}, automated detection of adversarial examples~\cite{hendrycks2017detecting,metzen2017detecting,manganet-defense}, application of various input
transformations~\cite{thermometer-defense,input-transformation-defense}, randomization~\cite{random-activation-pruning-defense,randomization-defense}, and  generative models~\cite{gan-defense,robust-manifold-defense,generative-models-defense}. Most of these defenses have been broken, sometimes months after their publication~\cite{carlini-attacks,carlini2017adversarial,obfuscated-gradients}.

The main empirical defense that still holds is Madry et al.~\cite{madry}, based
on adversarial training~\cite{goodfellow2014explaining}.  Madry et al. motivate
their approach with robust optimization, a rigorous theory. However not all the
assumptions are met, as this approach runs
a best-effort attack on each image in the minibatch at each gradient step, when
the theory requires finding the best possible adversarial attack.
And indeed, finding this worst case adversarial example for ReLU DNNs, used
in~\cite{madry}, was proven to be NP-hard in~\cite{2017arXiv171010571S}.
Therefore, while this defense works well in practice, it gives no theoretical guarantees for individual predictions or for the model’s accuracy under attack.
\pixeldp leverages DP theory to provide guarantees of robustness to arbitrary, norm-based attacks for individual predictions.

Randomization-based defenses are closest in method to our work~\cite{random-activation-pruning-defense,randomization-defense,random-self-ensemble}.
For example, Liu et al.~\cite{random-self-ensemble} randomizes the entire DNN and predicts using an ensemble of multiple copies of the DNN, essentially using draws to roughly estimate the expected $\arg\max$ prediction.
They observe empirically that randomization smoothens the prediction function, improving robustness to adversarial examples.
However, randomization-based prior work provides limited formalism that is insufficient to answer important defense design questions: where to add noise, in what quantities, and what formal guarantees can be obtained from randomization?
The lack of formalism has caused some works~\cite{random-activation-pruning-defense,randomization-defense} to add insufficient amounts of noise (e.g., noise not calibrated to pre-noise sensitivity), which makes them vulnerable to attack~\cite{carlini2017adversarial}.
On the contrary,~\cite{random-self-ensemble} inserts randomness into every layer of the DNN: our work shows that adding the right amount of calibrated noise at a single layer is sufficient to leverage DP's post-processing guarantee and carry the bounds through the end of the network.
Our paper formalizes randomization-based defenses using DP theory, and in doing so helps answer many of these design questions.
Our formalism also lets us reason about the guarantees obtained through randomization and enables us to elevate randomization-based approaches from the class of best-effort defenses to that of {\em certified defenses}.

\heading{Certified Defenses and Robustness Evaluations.}
\pixeldp offers two functions: (1) a strategy for learning robust models and (2) a method for evaluating the robustness of these models against adversarial examples.
Both of these approaches have been explored in the literature.
First, several certified defenses modify the neural network training process to minimize the number of robustness violations~\cite{wong2018provable, raghunathan2018certified, pmlr-v70-cisse17a}. These approaches, though promising, do not yet scale to larger networks like Google Inception~\cite{wong2018provable, raghunathan2018certified}.
In fact, all published certified defenses have been evaluated on small models and datasets~\cite{wong2018provable, raghunathan2018certified, pmlr-v70-cisse17a, mirman2018differentiable}, and at least in one case, the authors directly acknowledge that some components of their defense would be ``completely infeasible'' on ImageNet~\cite{wong2018provable}.
A recent paper \cite{2018arXiv180510265D} presents a certified defense evaluated on the CIFAR-10 dataset~\cite{cifar-dataset} for multi-layer DNNs (but smaller than ResNets). Their approach is completely different from ours and, based on the current results we see no evidence that it can readily scale to large datasets like ImageNet.

Another approach~\cite{2017arXiv171010571S} combines robust optimization and
adversarial training in a way that gives formal guarantees and has lower computational complexity than previous robust optimization work, hence it has the potential to scale better. This approach requires smooth DNNs (e.g., no ReLU or max pooling) and robustness guarantees are over the expected loss (e.g., log loss), whereas \pixeldp can certify each specific prediction, and also provides intuitive metrics like robust accuracy, which is not supported by \cite{2017arXiv171010571S}.
Finally, unlike \pixeldp, which we evaluated on five datasets of increasing size and complexity, this technique was evaluated only on MNIST, a small dataset that is notoriously amenable to robust optimization (due to being almost black and white).  Since the effectiveness of all defenses depends on the model and dataset, it is hard to conclude anything about how well it will work on more complex datasets.

Second, several works seek to formally verify~\cite{huang2017safety, DBLP:journals/corr/KatzBDJK17, wang2018efficient, reluval2018, dutta2018output, gehrai, tjeng2017evaluating} or lower bound~\cite{peck2017lower, weng2018evaluating} the robustness of pre-trained ML models against adversarial attacks.
Some of these works scale to large networks~\cite{peck2017lower, weng2018evaluating}, but they are insufficient from a defense perspective as they provide no scalable way to train robust models.

\heading{Differentially Private ML.}
Significant work focuses on making ML algorithms DP to preserve the privacy of
training sets~\cite{DBLP:conf/kdd/McSherryM09,2016arXiv160700133A,Chaudhuri:2011:DPE:1953048.2021036}.
PixelDP is orthogonal to these works, differing in goals, semantic, and algorithms.
The only thing we share with DP ML (and most other applied DP literature) are
DP theory and mechanisms.
The goal of DP ML is to learn the parameters of a model while ensuring DP with
respect to the training data. Public release of model parameters trained
using a DP learning algorithm (such as DP empirical risk minimization or ERM) is
guaranteed to not reveal much information about individual training examples.
PixelDP's goal is to create a robust predictive model where a small change to any
input example does not drastically change the model's prediction on that example.
We achieve this by ensuring that the model's scoring function is a DP function with respect
to the features of an input example (eg, pixels).
DP ML algorithms (e.g., DP ERM) do not necessarily produce models that satisfy
PixelDP's semantic, and our training algorithm for producing PixelDP models does
not ensure DP of training data.

\heading{Previous DP-Robustness Connections.}
Previous work studies generalization properties of
DP~\cite{bassily2016algorithmic}. It is shown that {\em learning algorithms}
that satisfy DP with respect to the training data have statistical benefits in
terms of out-of-sample performance; or that DP has a deep connection to robustness
at the dataset level~\cite{2013arXiv1306.1066D,dwork2009differential}.
Our work is rather different. Our learning algorithm is not DP; rather, the
predictor we learn satisfies DP with respect to the atomic units (e.g., pixels)
of a given test point.

\section{Conclusion}
\label{sec:conclusion}

We demonstrated a connection between robustness against adversarial
examples and differential privacy theory.
We showed how the connection can be leveraged to develop a certified
defense against such attacks that is (1) as effective at defending against $2$-norm attacks
as today's state-of-the-art best-effort defense and (2) more scalable and broadly applicable to large networks compared to any prior certified defense.
Finally, we presented the first evaluation of a certified $2$-norm defense on the large-scale ImageNet dataset.
In addition to offering encouraging results, the evaluation highlighted the substantial flexibility of our approach by leveraging a convenient autoencoder-based architecture to make the experiments possible with limited resources.

\section{Acknowledgments}

We thank our shepherd, Abhi Shelat, and the anonymous reviewers, whose comments helped us improve the paper significantly.
This work was funded through NSF CNS-1351089, CNS-1514437, and CCF-1740833, ONR N00014-17-1-2010, two Sloan Fellowships, a Google Faculty Fellowship, and a Microsoft Faculty Fellowship.

% -------------------- %

\begin{spacing}{0.92}
{
  \bibliographystyle{abbrv}
  % \bibliography{bib/abbrev,bib/conferences,bib/refs,bib/news}

}
\end{spacing}

\appendix
\section{Appendix}

\subsection{Proof of Proposition~\ref{prop:general-robustness-condition}}
\label{appendix:proposition-2-proof}

We briefly re-state the Proposition and detail the proof.
\begin{proposition*}
  Suppose $A$ is $(\epsilon,\delta)$-PixelDP for
  size $L$ in $p$-norm metric.
  For any input $x$, if for some $k \in \mathcal{K}$,
  \[
  \hat\E^{lb}(A_k(x))
  > e^{2\epsilon} \max_{i : i \neq k} \hat\E^{ub}(A_i(x)) + (1+e^{\epsilon}) \delta ,
  \]
  then the multiclass classification model based on label probabilities $\hat\E(A_k(x))$ is robust to attacks of $p$-norm $L$ on input $x$ with probability higher than $\eta$.
\end{proposition*}
\begin{proof}
  Consider any $\alpha \in B_p(L)$, and let $x' := x + \alpha$.
  From Equation~\eqref{eq:dp}, we have with $p > \eta$ that
  \begin{align*}
    \hat\E(A_{k}(x')) &  \ge (\hat\E(A_{k}(x)) - \delta) / e^\epsilon \\
                      &  \ge (\hat\E^{lb}(A_{k}(x)) - \delta) / e^\epsilon , \\
    \hat\E(A_{i : i \neq k}(x')) & \leq e^\epsilon \max_{i : i \neq k} \hat\E^{ub}(A_{i}(x)) + \delta, \quad i \neq k .
  \end{align*}
  Starting from the first inequality, and using the hypothesis, followed by the second
  inequality, we get
  \begin{align*}
   \hat\E^{lb}(A_{k}(x)) > \ & \ e^{2\epsilon} \max_{i : i \neq k} \hat\E^{ub}(A_i(x)) + (1 + e^\epsilon) \delta \Rightarrow \\
    \hat\E(A_{k}(x')) \geq \ & (\hat\E^{lb}(A_{k}(x)) - \delta) / e^\epsilon \\
                                  > \ & e^{\epsilon} \max_{i : i \neq k} \hat\E^{ub}(A_i(x)) + \delta \\
               > \ & \hat\E(A_{i : i \neq k}(x'))
  \end{align*}
  which is the robustness condition from Equation~\eqref{eq:robustness}.
\end{proof}

\subsection{Design Choice}
\label{appendix:design-choices}

Our theoretical results allow the DP DNN to output any bounded score over labels $A_{k}(x)$. In the evaluation we used the softmax output of the DNN, the typical ``probabilities'' that DNNs traditionally output. We also experimented with using $\arg\max$ scores, transforming the probabilities in a zero vector with a single $1$ for the highest score label. As each dimension of this vector is in $[0, 1]$, our theory applies as is. We observed that $\arg\max$ scores were a bit less robust empirically (lower accuracy under attack). However, as shown on \F\ref{fig:argmax-robust-acc} $\arg\max$ scores yield a higher certified accuracy. This is both because we can use tighter bounds for measurement error using a Clopper-Pearson interval, and because the $\arg\max$ pushes the expected scores further apart, thus satisfying Proposition~\ref{prop:robustness-condition} more often.

\begin{figure}[t]
\centering
  \includegraphics[width=0.40\textwidth]{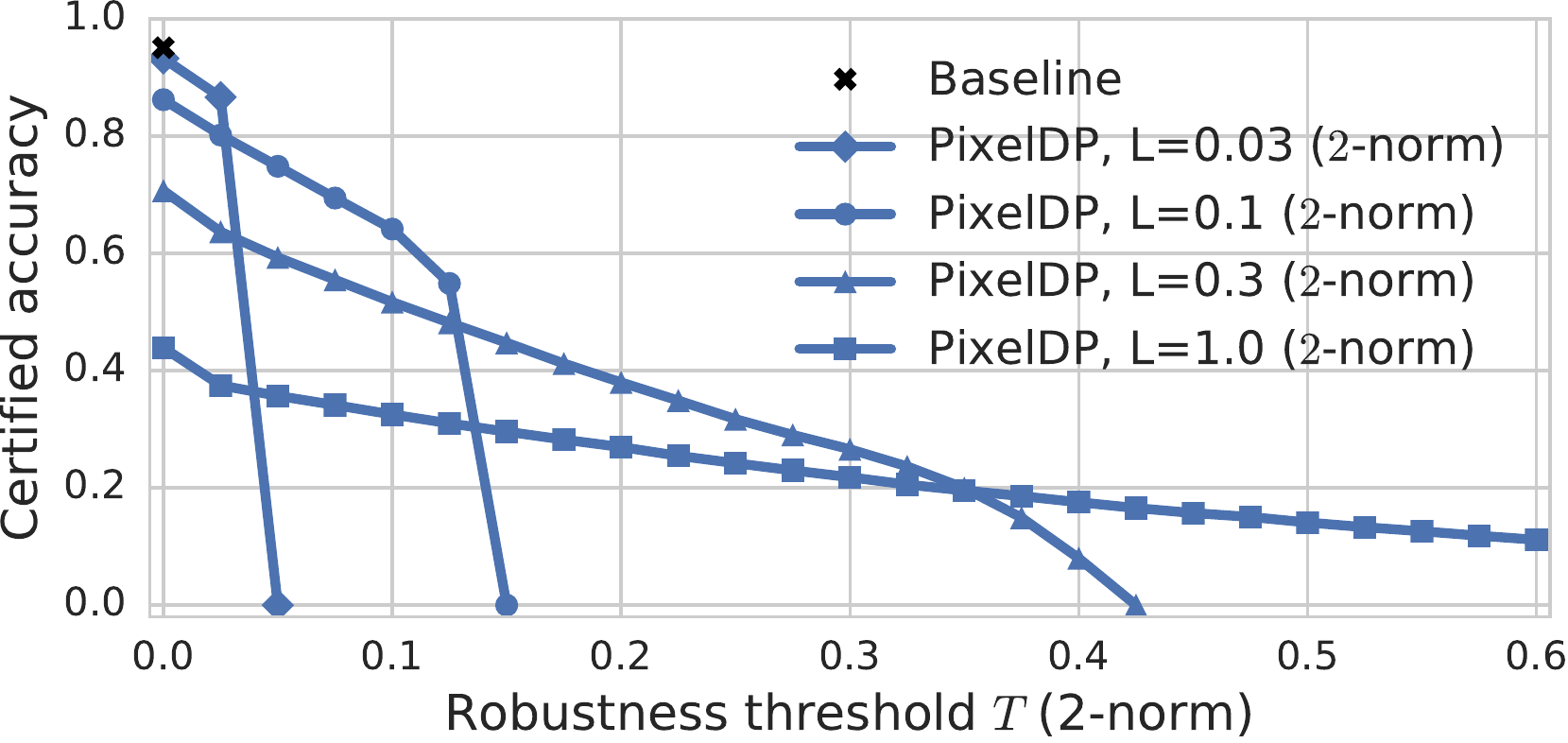}
  \vspace{-2pt}
\caption{{\bf Robust Accuracy, $\arg\max$ Scores.}
  Using $\arg\max$ scores for certification yields better accuracy bounds (see \F\ref{sec:l2-robust-accuracy-cifar}), both because the scores are further appart and because the measurement error bounds are tighter.}
\label{fig:argmax-robust-acc}
\vspace{-10pt}
\end{figure}

We also study the impact of the DP mechanism used on certified accuracy for $1$-norm attacks.
Both the Laplace and Gaussian mechanisms can be used after
the first convolution, by respectively controlling the $\Delta_{1,1}$ or $\Delta_{1,2}$
sensitivity.
\F\ref{fig:laplace-vs-gaussian} shows that for our ResNet, the Laplace mechanism
is better suited to low levels of noise: for $L=0.1$, it yields a slightly higher accuracy
($90.5\%$ against $88.9\%$), as well as better certified accuracy with a maximum
robustness size of $1$-norm $0.22$ instead of $0.19$, and a robust accuracy of
$73\%$ against $65.4\%$ at the $0.19$ threshold.
On the other hand, when adding more noise (e.g. $L=0.3$), the Gaussian mechanism
performs better, consistently yielding a robust
accuracy $1.5$ percentage point higher.

\begin{figure}[t]
\centering
  \includegraphics[width=0.40\textwidth]{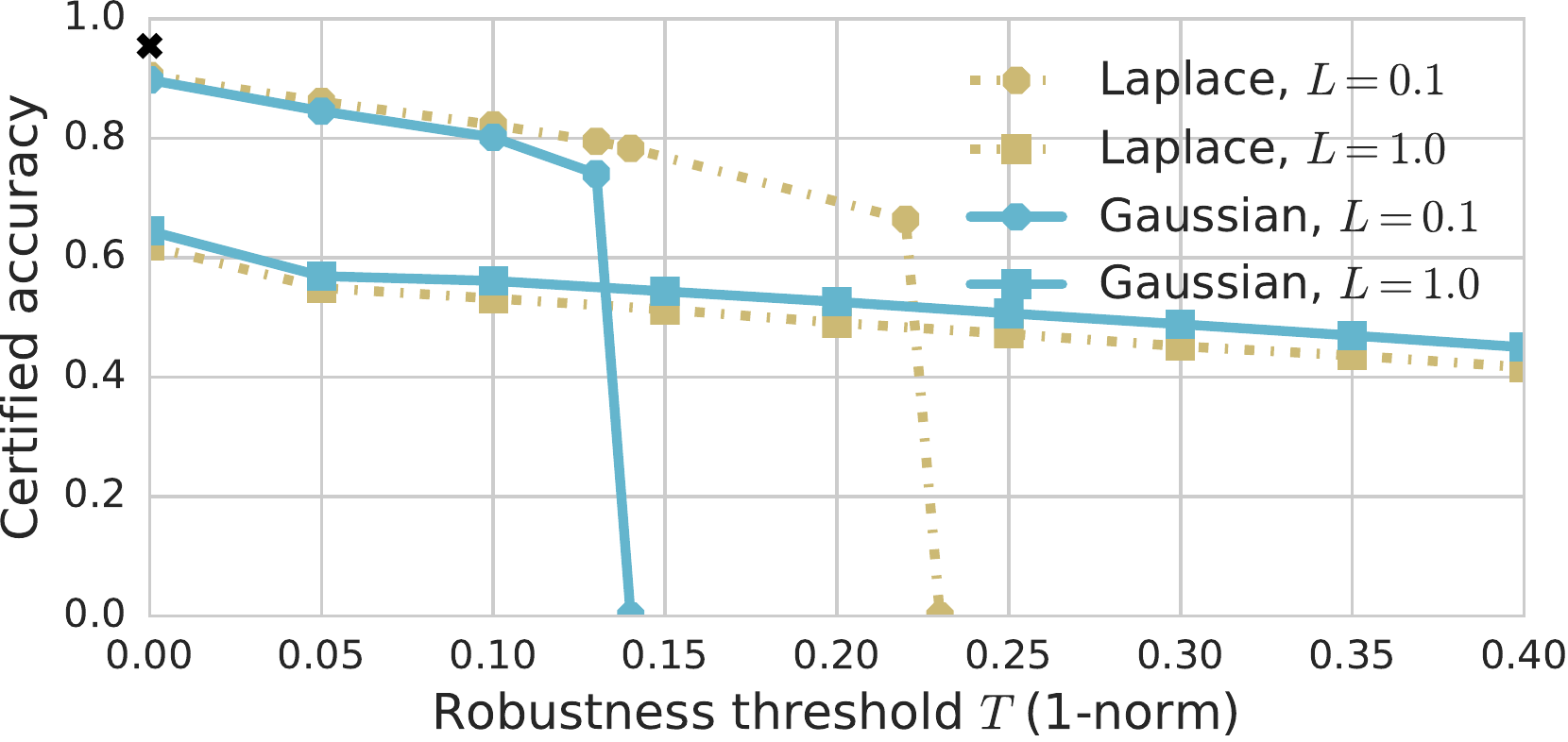}
  \vspace{-2pt}
\caption{{\bf Laplace vs. Gaussian.}
Certified accuracy for a ResNet on the CIFAR-10 dataset, against $1$-norm
  bounded attacks.  The Laplace mechanism yields better accuracy for low noise
  levels, but the Gaussian mechanism is better for high noise ResNets.
}
\label{fig:laplace-vs-gaussian}
\vspace{-15pt}
\end{figure}

\subsection{Attack Details}
\label{appendix:pgd-attack-details}

All evaluation results (\S\ref{sec:evaluation}) are based on the attack from
Carlini and Wagner~\cite{carlini-attacks}, specialized to better attack \pixeldp
(see parameters and adaptation in \S\ref{sec:evaluation-methodology}). We also
implemented variants of the iterative Projected Gradient Descent (PGD) attack
described in~\cite{madry}, modified to average gradients over $15$ noise draws
per step, and performing each attack $15$ times with a small random initialization.
We implemented two version of this PGD attack.

{\em 2-norm Attack:} The gradients are normalized before applying the step
size, to ensure progress even when gradients are close to flat. We perform
$k=100$ gradient steps and select a step size of $\frac{2.5L}{k}$. This
heuristic ensures that all feasible points within the 2-norm ball can be reached
after $k$ steps.  After each step, if the attack is larger than $L$, we project
it on the 2-norm ball by normalizing it. Under this attack, results were
qualitatively identical for all experiments. The raw accuracy numbers were a
few percentage points higher (i.e. the attack was slightly less efficient), so
we kept the results for the Carlini and Wagner attack.

{\em $\infty$-norm Attack:} We perform $max(L+8, 1.5L)$ gradient steps and
maintain a constant step of size of $0.003$ (which corresponds to the minimum pixel
increment in a discrete $[0, 255]$ pixel range). At the end of each gradient
step we clip the size of the perturbation to enforce a perturbation within the
$\infty$-norm ball of the given attack size.
We used this attack to compare \pixeldp with models from Madry and RobustOpt
(see results in Appendix \ref{appendix:linf-attacks}).

Finally, we performed sanity checks suggested in~\cite{obfuscated-gradients}.
The authors observe that several heuristic defenses do not ensure the absence of
adversarial examples, but merely make them harder to find by obfuscating
gradients. This phenomenon, also referred to as gradient masking
\cite{sok-towards-the-science-of-security-and-privacy, tramer2017ensemble},
makes the defense susceptible to new attacks crafted to circumvent that
obfuscation~\cite{obfuscated-gradients}.
Although \pixeldp provides certified accuracy bounds that are {\em guaranteed}
to hold regardless of the attack used, we followed guidelines from
\cite{obfuscated-gradients}, to to rule out obfuscated gradients in our
empirical results. We verified three properties that can be symptomatic of
problematic attack behavior.  First, when growing $T$, the accuracy drops to $0$
on all models and datasets.  Second, our attack significantly outperforms random
sampling.  Third, our iterative attack is more powerful than the respective
single-step attack.

\subsection{$\infty$-norm Attacks}
\label{appendix:linf-attacks}

\begin{figure}[t]
\centering
  \includegraphics[width=0.40\textwidth]{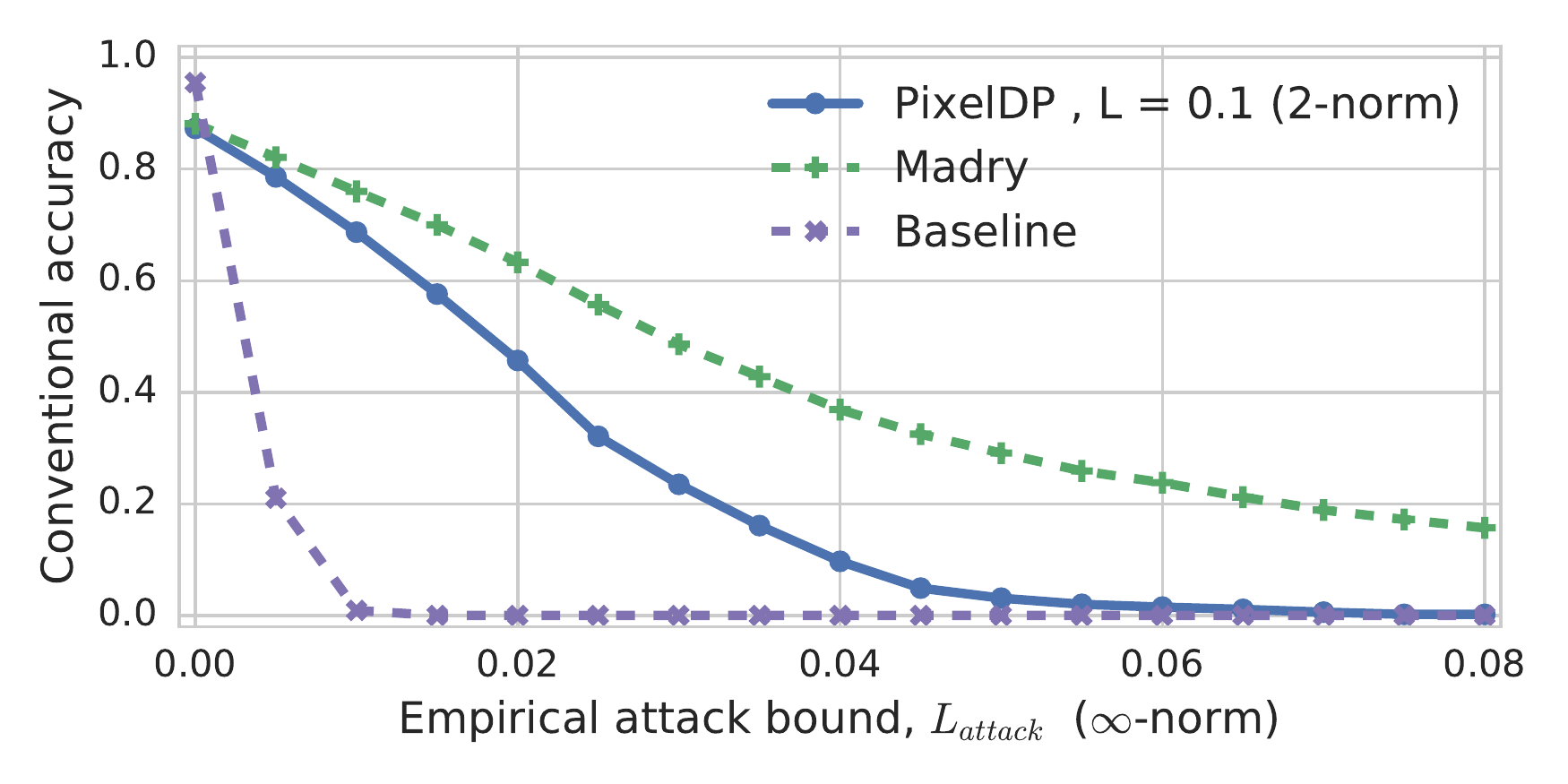}
  \vspace{-5pt}
\caption{{\bf Accuracy under $\infty$-norm attacks for \pixeldp and Madry.}
The Madry model, explicitly trained against $\infty$-norm attacks, outperforms \pixeldp. The difference increases with the size of the attack.
}
\label{fig:linf-madry}
\vspace{-15pt}
\end{figure}

\begin{figure}[t]
\centering
  \includegraphics[width=0.40\textwidth]{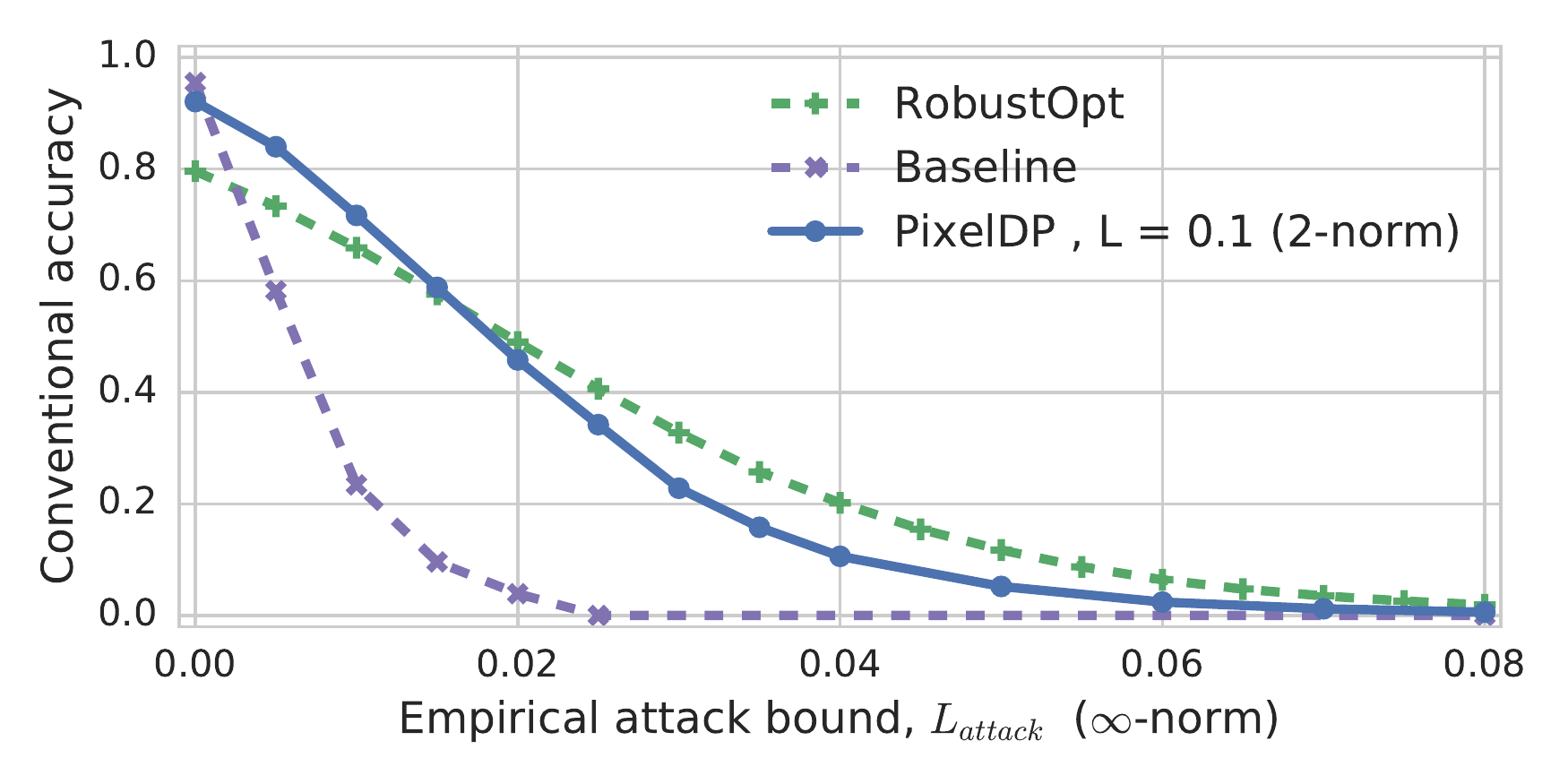}
  \vspace{-5pt}
\caption{{\bf Accuracy under $\infty$-norm attacks for \pixeldp and RobustOpt.}
\pixeldp is better up to $L_\infty=0.015$, due to its support of larger ResNet models. For attacks of $\infty$-norm above this value, RobustOpt is more robust.
}
\label{fig:linf-cmu}
\vspace{-15pt}
\end{figure}

As far as $\infty$-norm attacks are concerned, we acknowledge that the size
of the attacks against which our current \pixeldp defense can certify accuracy is
substantially lower than that of previous certified defenses.
Although previous defenses have been demonstrated on MNIST and SVHN only, and for
smaller DNNs, they achieve $\infty$-norm
defenses of $T_\infty=0.1$ with robust accuracy $91.6\%$
\cite{wong2018provable} and $65\%$
\cite{raghunathan2018certified} on MNIST. On SVHN, \cite{wong2018provable}
uses $T_\infty=0.01$, achieving 59.3\% of certified accuracy.
Using the crude bounds we have between $p$-norms makes a comparison difficult in both directions. Mapping
$\infty$-norm bounds in $2$-norm gives $T_2 \geq T_\infty$, also yielding very small bounds.
On the other hand, translating $2$-norm guarantees into
$\infty$-norm ones (using that $\|x\|_2 \leq \sqrt{n} \|x\|_\infty$ with $n$ the
size of the image) would require a $2$-norm defense of size $T_2=2.8$ to match the $T_\infty=0.1$ bound from MNIST, an order of magnitude higher than what we can achieve.
As comparison points, our $L=0.3$ CNN has a robust accuracy of $91.6\%$ at
$T=0.19$ and $65\%$ at $T=0.39$.
We make the same observation on SVHN, where we would need a bound at $T_2=0.56$ to match the $T_\infty=0.01$ bound, but our ResNet with $L=0.1$ reaches a similar robust accuracy as RobustOpt for $T_2=0.1$.
This calls for the design $\infty$-norm specific \pixeldp mechanisms that could also scale to larger DNNs and datasets.

On Figures~\ref{fig:linf-madry} and~\ref{fig:linf-cmu}, we show \pixeldp's
accuracy under $\infty$-norm attacks, compared to the Madry and RobustOpt
models, both trained specifically against this type of attacks.
On CIFAR-10, the Madry model outperforms \pixeldp: for $L_{attack}=0.01$,
\pixeldp's accuracy is 69\%, 8 percentage points lower than Madry's. The gap increases
until \pixeldp arrives at $0$ accuracy for $L_{attack}=0.06$, with Madry still having
22\%.

On SVHN, against the RobustOpt model, trained with robust optimization against $\infty$-norm attacks, \pixeldp is better up to $L_\infty=0.015$, due to its support of larger ResNet models. For attacks of $\infty$-norm above this value, RobustOpt is more robust.

\subsection{Extension to regression}
\label{appendix:regression-extension}

A previous version of this paper contained an incorrect claim in the statement of Lemma~\ref{lemma:expectation-bound} for outputs that can be negative. Because the paper focused on classification, where DNN scores are in $[0,1]$, the error had no impact on the claims or experimental results for classification.
Lemma~\ref{lemma:general-expectation-bound}, below, provides a correct version of Lemma~\ref{lemma:expectation-bound} for outputs that can be negative, showing how PixelDP can be extended to support regression problems:

\begin{lemma} \label{lemma:general-expectation-bound}
	{\bf (General Expected Output Stability Bound)}
	Suppose a randomized function $A$, with bounded output $A(x) \in [a,b], \ a,b \in \mathbb{R}$, with $a \leq 0 \leq b$, satisfies $(\epsilon,\delta)$-DP.
  Let $A_+(x) = \max(0, A(x))$ and $A_-(x) = -\min(0, A(x))$, so that $A(x) = A_+(x) - A_-(x)$.
  Then the expected value of its output meets the following property:
  for all $\alpha \in B_p(1)$,
  \begin{align*}
    & \E(A(x + \alpha)) \leq e^\epsilon \E(A_+(x)) - e^{-\epsilon} \E(A_-(x)) + b\delta - e^{-\epsilon} a\delta , \\
    & \E(A(x + \alpha)) \geq e^{-\epsilon} \E(A_+(x)) - e^{\epsilon} \E(A_-(x)) - e^{-\epsilon}b\delta + a\delta .
  \end{align*}
	The expectation is taken over the randomness in $A$.
\end{lemma}
\begin{proof}
	Consider any $\alpha \in B_p(1)$, and let $x' := x + \alpha$.
  Observe that $\E(A_+(x)) = \int_0^b P(A(x)>t) dt$, so
  by the $(\epsilon, \delta)$-DP property of $A$ via Equation~\eqref{eq:dp},
  we have $\E(A_+(x')) \leq e^\epsilon \E(A_+(x)) + b\delta$ and
  $\E(A_+(x')) \geq e^{-\epsilon} \E(A_+(x)) - e^{-\epsilon} b\delta$.
  Similarly,
  $\E(A_-(x')) \leq e^\epsilon \E(A_-(x)) - a\delta$ and
  $\E(A_-(x')) \geq e^{-\epsilon} \E(A_+(x)) + e^{-\epsilon} a\delta$.
  Putting these four inequalities together concludes the proof.
\end{proof}

Following Lemma~\ref{lemma:general-expectation-bound}, supporting regression problems involves three steps. First, if the output is unbounded, one must use $(\epsilon, 0)$-DP (e.g. with the Laplace mechanism). If the output is bounded, one may use $(\epsilon, \delta)$-DP. The output may be bounded either naturally, because the specific task has inherent output bounds, or by truncating the results to a large range of values and using a comparatively small $\delta$.

Second, instead of estimating the expected value of the randomized prediction function, we estimate both $A_+(x)$ and $A_-(x)$. We can use Hoeffding's inequality \cite{hoeffding1963probability} or Empirical Bernstein bounds \cite{DBLP:conf/colt/MaurerP09} to bound the error.

Third, following Lemma~\ref{lemma:general-expectation-bound}, we bound $A_+(x)$ and $A_-(x)$ separately using the DP Expected Output Stability Bound, to obtain a bound on $\E(A(x)) = \E(A_+(x)) - \E(A_-(x))$.

\end{document}